\def\year{2020}\relax
\documentclass[letterpaper]{article} 
\usepackage{aaai20}  
\usepackage[utf8]{inputenc} 
\usepackage{booktabs}       
\usepackage{amsfonts}       
\usepackage{nicefrac}       
\usepackage{microtype}      
\usepackage{amsmath,amsthm,amssymb}
\usepackage{paralist}
\usepackage{dsfont}
\usepackage{graphicx}
\usepackage{ragged2e}
\usepackage{algorithm}
\usepackage[noend]{algpseudocode} 
\usepackage{comment}
\usepackage{tikz,pgf}
\usepackage{appendix}
\usepackage{natbib}

\newcommand{\ww}{\mathbf{w}}
\newcommand{\WW}{\mathbf{W}}
\newcommand{\xx}{\mathbf{x}}

\newcommand{\fc}{ \thinspace C_{t} + R_{1t} \thinspace }
\renewcommand{\sc}{ \thinspace R_{2t} + M_{t} \thinspace }

\newtheorem{theorem}{Theorem}
\newtheorem{lemma}[theorem]{Lemma}
\newtheorem{definition}[theorem]{Definition}

\newtheorem{corollary}[theorem]{Corollary}

\def\BState{\State\hskip-\ALG@thistlm}

\def \K {\mathcal{K}}
\def \X {\mathcal{X}}
\def \R {\mathbb{R}}
\def \RR {\mathbf{R}}
\def \I {\mathbb{I}}
\def \E {\mathbb{E}}
\def \ms {\sigma ( yf( \xx ) - \rho )}
\def \ps {\sigma ( yf( \xx ) + \rho )}
\usepackage{times}  
\usepackage{helvet} 
\usepackage{courier}  
\usepackage[hyphens]{url}  
\usepackage{graphicx}  
\title{Online Active Learning of Reject Option Classifiers}
\author{Kulin Shah, Naresh Manwani \\
Machine Learning Lab, KCIS, IIIT Hyderabad, India \\
kulin.shah@students.iiit.ac.in, naresh.manwani@iiit.ac.in
}

\begin{document}

\maketitle

\begin{abstract}
  Active learning is an important technique to reduce the number of labeled examples in supervised learning. Active learning for binary classification has been well addressed in machine learning. However, active learning of the reject option classifier remains unaddressed. In this paper, we propose novel algorithms for active learning of reject option classifiers. We develop an active learning algorithm using double ramp loss function. We provide mistake bounds for this algorithm. We also propose a new loss function called double sigmoid loss function for reject option and corresponding active learning algorithm. We offer a convergence guarantee for this algorithm. We provide extensive experimental results to show the effectiveness of the proposed algorithms. The proposed algorithms efficiently reduce the number of label examples required.
\end{abstract}

\section{Introduction}
In standard binary classification problems, algorithms return prediction on every example. For any misprediction, the algorithms incur a cost. Many real-life applications involve very high misclassification costs. Thus, for some confusing examples, not predicting anything may be less costly than any misclassification. The choice of not predicting anything for an example is called {\em reject option} in machine learning literature. Such classifiers are called reject option classifiers.

Reject option classification is very useful in many applications. Consider a doctor diagnosing a patient based on the observed symptoms and preliminary diagnosis. If there is an ambiguity in observations and preliminary diagnosis, the doctor can hold the decision on the treatment.  She can recommend to take advanced tests or consult a specialist to avoid the risk of misdiagnosing the patient. The holding response of the doctor is the same as to reject option for the specific patient \citep{Rocha2011}. On the other hand, the doctor's misprediction can cost huge money for further treatment or the life of a person. In another example, a banker can use the reject option while looking at the loan application of a customer \citep{Rosowsky2013}. A banker may choose not to decide based on the information available because of high misclassification cost, and asks for further recommendations or a credit bureau score from the stakeholders. Application of reject option classifiers include healthcare \cite{btn349,Rocha2011}, text categorization \cite{1234113}, crowdsourcing \cite{Qunwei2017} etc.

Let $\X \subset \R^d$ be the feature space and $\{+1,-1\}$ be the label space. Examples of the form $(\xx,y)$ are generated from an unknown fixed distribution on $\X \times \{+1,-1\}$. A reject option classifier can be described with the help of a function $f:\X \rightarrow \R$ and a rejection width parameter $\rho \in \R_+$ as below.
\begin{align}
    \label{eq:reject_classifier}
    h_\rho(f(\xx)) = 1.\I_{\{f(\xx)>\rho\}} -1.\I_{\{f(\xx)<-\rho\}}-0.\I_{\{ | f(\xx) | \leq \rho\}}
\end{align}
 The goal is to learn $f(.)$ and $\rho$ simultaneously. For a given example $(\xx,y)$, the performance of reject option classifier $h_\rho(f(.))$ is measured using following loss function.
 \begin{align}
   \label{eq:0-d-1}
   L_d(yf(\xx),\rho) = \I_{\{yf(\xx) \leq -\rho\}} + d\I_{\{|f(\xx)|\leq \rho\}}
 \end{align}
 where $d\in (0,0.5)$ is the cost of rejection. A reject option classifier is learnt by minimizing the risk (expectation of loss) under $L_d$. As $L_d$ is not continuous, optimization of empirical risk under $L_d$ is difficult.  \citet{Bartlett:2008, wegkamp2011} propose a convex surrogate of $L_d$ called generalized hinge loss. They learn the reject option classifier using risk minimization algorithms based on generalized hinge loss. \citet{Grandvalet2008} propose another convex surrogate of $L_d$ called double hinge loss and corresponding risk minimization approach for reject option classification. \citet{Manwani15,shah2019sparse} propose double ramp loss based approaches for reject option classification. Double ramp loss is a non-convex bounded loss function. All these approaches assume that we have plenty of labeled data available.

In general, classifiers learned with a large amount of training data can give better generalization on testing data. However, in many real-life applications, it can be costly and difficult to get a large amount of labeled data. Thus, in many cases, it is desirable to ask the labels of the examples selectively. This motivates the idea of active learning. Active learning selects more informative examples and queries labels of those examples.  Active learning of standard binary classifiers has been well-studied \citep{Dasgupta:2009, Bachrach:1999, Tong:2002}. In \citet{El-Yaniv:2012}, authors reduce active learning for the usual binary classification problem to learning a reject option classifier to achieve faster convergence rates. However, active learning of reject option classifiers has remained an unaddressed problem. In this paper, we propose online active learning algorithms to reject option classification.

Let us reconsider the example where the banker uses the reject option classifier for selecting the loan applications. Consider a loan application that satisfies the basic requirements. Thus, the banker is not clear about using the hold option. On the other hand, she is also not sure enough to approve the application. Such cases are instrumental in defining the separation rule between accepting the loan application and holding it for further investigation. 
This motivates us to think that one can use active learning to ask the labels of selective examples as described above while learning the reject option classifier.

A broad class of active learning algorithms is inspired by the concept of a margin between the two categories. Thus, an example, which falls in the margin area of the current classifier, carries more information about the decision boundary. On the other hand,  examples which are correctly classified with good margin or misclassified by a good margin, give less knowledge of the decision boundary. Margin examples can bring more changes to the existing classifier. Thus, querying the label of margin examples is more desirable than the other two kinds of examples. 

A reject option classifier can be viewed as two parallel surfaces with the rejection area in between. Thus, active learning of the reject option classifier becomes active learning of two surfaces in parallel with a shared objective. This shared objective is nothing but to minimize the sum of $L_d$ losses over a sequence of examples. 
 In \cite{Manwani15}, the authors propose a risk minimization approach based on double ramp loss ($L_{dr}$) for learning the reject option classifier. In \cite{Manwani15}, it is shown that at the optimality, the two surfaces can be represented using only those examples which are close to them. Examples that are far from the two surfaces do not participate in the representation of the surfaces. This motivates us to use double ramp loss for developing an active learning approach to reject option classifiers. 

\subsection{Our Contributions}
We make the following contributions in this paper. 
\begin{itemize}
    \item We propose an active learning algorithm based on double ramp loss $L_{dr}$ to learn a linear and non-linear classifier. We give bounds to the number of rejected examples and misclassification rates for un-rejected examples.
    \item We propose a smooth non-convex loss called double sigmoid loss ($L_{ds}$) for reject option classification.
    \item We propose an active learning algorithm based on $L_{ds}$ to learn both linear and non-linear classifiers. We also give convergence guarantees for the proposed algorithm.
    \item We present extensive simulation results for both proposed active learning algorithms for linear as well as non-linear classification boundaries. 
\end{itemize}


\section{Proposed Approach: Active Learning Inspired by Double Ramp Loss}
\label{section:dr_active}
Active learning algorithm does not ask the label in every trial. We denote the instance presented to algorithm at trial $t$ by $\xx_{t}$. Each $\xx_{t} \in \mathcal{X}$ is associated with a unique label $y_{t} \in \{-1, 1\}$. The algorithm calculates $f_{t}(\xx_{t})$ and outputs the decision using eq.(\ref{eq:reject_classifier}).  Based on $f_{t}(\xx_{t})$, the active learning algorithm decides whether to ask label or not.  
\cite{pmlr-v5-guillory09a} shows that online active learning algorithms can be viewed as stochastic gradient descent on non-convex loss function therefore, we use a non-convex loss function {\em Double ramp loss} $L_{dr}$ \citep{Manwani15} to derive our first active learning approach. $L_{dr}$ is defined as follows.   
\begin{equation*}
\begin{aligned}
 L_{dr}(yf(\xx), {\rho}) = d \Big{[}\big{[}1-yf(\xx)
+\rho\big{]}_+ - \big{[}-1-yf(\xx)+\rho\big{]}_+\Big{]}\\
+(1-d) \; \Big{[}\big{[}1
-yf(\xx)-\rho\big{]}_+ - \big{[}-1-yf(\xx)-\rho\big{]}_+\Big{]}
\end{aligned}
\end{equation*}
\begin{figure}
    \centering
    \includegraphics[width=0.7\columnwidth]{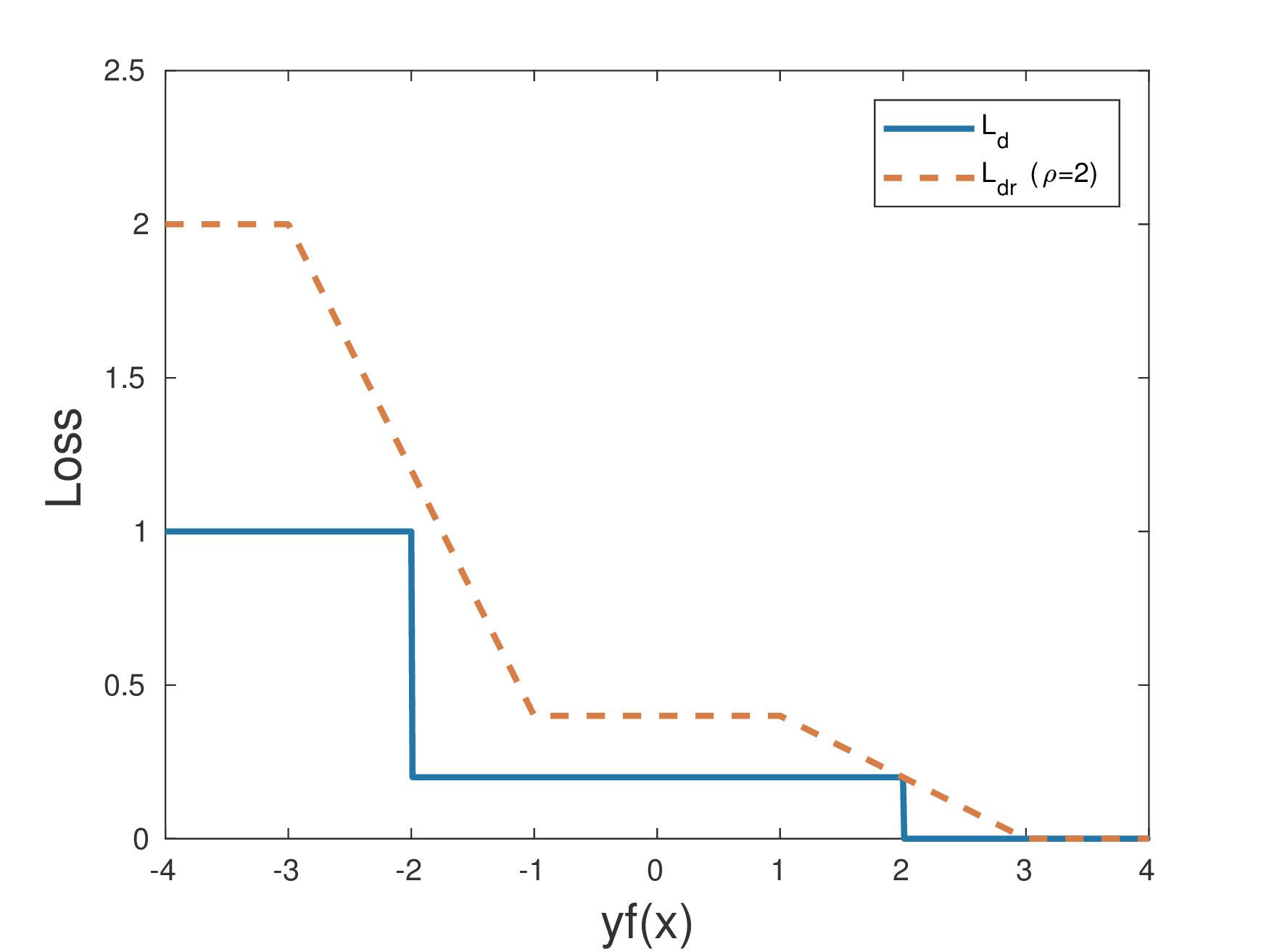}
    \caption{Double Ramp Loss with $\rho=2$}
    \label{fig:my_label}
\end{figure}
Here $[a]_{+} = \max(0, a)$ and $d$ is the cost of rejection. Figure~\ref{fig:my_label} shows the plot of double ramp loss for $\rho=2$.

We first consider developing active learning algorithm for linear classifiers (i.e. $f(\xx) = \ww \cdot \xx$). We use stochastic gradient descent (SGD) to derive double ramp loss based active learning algorithm. 
Parameters update equations using SGD are as follows. 
\begin{align*}
&\;\;\;\ww_{t+1} = \ww_t - \eta \nabla_{\ww_t} L_{dr}(y_tf(\xx_t), \rho_{t})\\
=&\begin{cases}
\ww_{t} + \eta dy_t\xx_t,           &  \rho_t - 1 \leq y_tf(\xx_t) \leq \rho_t + 1 \\
\ww_{t} + \eta (1-d)y_t\xx_t        &  -\rho_t - 1 \leq y_tf(\xx_t) \leq -\rho_t+1 \\
\ww_{t}                             & \text{otherwise}
\end{cases}\\
{\rho}_{t+1}& = \rho_t - \eta \nabla_{\rho_t} L_{dr}(y_tf(\xx_t), \rho_{t})\\
&=
\begin{cases}
{\rho}_{t} - \eta d,         &  \rho_t - 1 \leq y_tf(\xx_t) \leq \rho_t + 1 \\
{\rho}_{t} + \eta (1-d),    &  -\rho_t - 1 \leq y_tf(\xx_t) \leq -\rho_t + 1 \\
{\rho}_{t}                         & \text{otherwise}
\end{cases}
\end{align*}
Where $\eta$ is the step-size. We see that the parameters are updated only when $|f_t(\xx_t)| \in [\rho_t - 1, \rho_t + 1]$. For the rest of the regions, the gradient of the loss $L_{dr}$ is zero therefore, there won't be any update when an example $\xx_t$ is such that $|f_t(\xx_t)| \notin [\rho_t - 1, \rho_t + 1]$. Thus, there is no need to query the label when $|f_t(\xx_t)| \notin [\rho_t - 1, \rho_t + 1]$. We only query the labels when $|f_t(\xx_t)| \in [\rho_t - 1, \rho_t + 1]$. Thus, we ask the label of the current example only if it falls in the linear region of the loss $L_{dr}$. This is the same way any margin based active learning approach updates the parameters. If the algorithm does not query the label $y_t$, the parameters ($\ww, \rho$) are not updated. Thus, we define the query function $Q_t$ as follows.
\begin{equation}
Q_t=
\begin{cases}
1 & \text{if} \; \rho_t - 1 \leq |f(\xx_t)| \leq \rho_t + 1 \\
0 & \text{otherwise}
\end{cases}
\end{equation}
 The detailed algorithm is given in Algorithm~\ref{algo:double-ramp-active-learning}. We call it DRAL (double ramp loss based active learning). 
DRAL can be easily extended for learning nonlinear classifiers using kernel trick and is described in Appendix \ref{app-sec:kernelized-DRAL}. 
\begin{algorithm}
\caption{{\bf D}ouble {\bf R}amp Loss {\bf A}ctive {\bf L}earning (DRAL)}
\label{algo:double-ramp-active-learning}
\begin{algorithmic}
\State {\bf Input:} $d \in (0, 0.5)$, step size $\eta$
\State {\bf Output:} Weight vector $\ww$, Rejection width $\rho$
\State {\bf Initialize:} $\ww_1=\mathbf{0}, \rho_{1}=1$
\For{$t = 1,\ldots,T$} 
\State Sample $\xx_{t} \in S$ 
\State Set $f_{t}(\xx_{t})=\ww_{t} \cdot \xx_{t}$
\If{$\rho_{t} -1 \leq |f_{t}(\xx_{t})| \leq \rho_{t} + 1$}
\State Set $Q_t=1$
\State Query the label $y_{t}$ of $\xx_t$.
\If{($\rho_{t} - 1 \leq \; y_{t}f_{t}(\xx_t) \; \leq \rho_{t} + 1$)}
\State $\ww_{t+1} = \ww_{t} + \eta dy_{t}\xx_{t}$.
\State $\rho_{t+1} = \rho_{t} - \eta d$
\ElsIf{($-\rho_{t} - 1 \leq y_{t}f_{t}(\xx_{t}) \leq -\rho_{t} + 1$)}
\State $\ww_{t+1} = \ww_{t} + \eta (1-d)y_{t}\xx_{t}$
\State $\rho_{t+1} = \rho_{t} + \eta (1-d)$
\EndIf
\Else 
\State $\ww_{t+1} = \ww_{t}$ 
\State $\rho_{t+1} = \rho_{t}$
\EndIf
\EndFor
\end{algorithmic}
\end{algorithm}
 
\subsection{Mistake Bounds for DRAL}
In this section, we derive the mistake bounds of DRAL. Before presenting the mistake bounds, we begin by presenting a lemma which would facilitate the following mistake bound proofs. Let $f_t(\xx_t)=\ww_t \cdot \xx_t$. We define the following.\footnote{$\mathbb{I}_{\{A\}}$ takes value 1 when $A$ is true and 0 otherwise.}
\begin{align}
\label{define-ct-rt-mt}
\begin{cases}
C_{t} = \mathbb{I}_{ \{ {\rho}_{t} \leq y_{t}f_t(\xx_{t}) \leq {\rho}_{t} + 1 \} }&R_{1t} = \mathbb{I}_{ \{ {\rho}_{t} - 1 \leq y_{t}f_t(\xx_{t}) \leq {\rho}_{t} \} } \\
R_{2t} = \mathbb{I}_{ \{ -{\rho}_{t}  \leq y_{t}f_t(\xx_{t}) \leq -{\rho}_{t} + 1 \} }&M_{t} = \mathbb{I}_{ \{ -{\rho}_{t} - 1  \leq y_{t}f_t(\xx_{t}) \leq -{\rho}_{t} \} } 
\end{cases}
\end{align}
\begin{lemma}
\label{lemma-first-lemma-general-results}
Let $(\xx_{1}, y_{1}), \dotso , (\xx_{T}, y_{T})$ be a sequence of input instances, where $\xx_{t} \in \mathcal{X}$ and $y_{t} \in \{-1, 1\}$ for all $t\in[T]$.\footnote{Here, $[T]$ denotes the sequence $1,\ldots, T$.} Given $C_{t}, R_{1t}, R_{2t} \text{ and } M_{t}$ as defined in eq.(\ref{define-ct-rt-mt}) and $\alpha >0$, the following bound holds for any $\ww$ such that $\| \ww \| \leq \WW$.

\begin{equation*}
\begin{aligned}
 &{\alpha}^2{ \| \ww \| }^2 + (1 - {\alpha}\rho)^2 + \frac{ 2 \alpha \eta }{ m }  \sum_{t=1}^T  L_{dr}(y_tf(\xx_{t}), \rho) \geq \\
 &\sum\limits_{t=1}^{T} [\fc]\big{[}  2{\alpha} \eta d + 2 \eta (L_{dr}(y_tf_t(\xx_{t}), \rho_t) - d )  \\
 -& \eta^2 d^2( {\| \xx_{t} \|}^2 + 1 ) \big{]}+ \sum\limits_{t=1}^{T} [\sc]\big{[}   \frac{ 2{\alpha} \eta (1 + d) m_{21} }{ m_{22} }  \\
 +& 2 \eta (L_{dr}(y_tf_t(\xx_{t}), \rho_t) - d- 1) - \eta^2 (1-d)^2( {\| \xx_{t} \|}^2 + 1 )  \big{]} 
\end{aligned}
\end{equation*}
where $f(\xx_t)=\ww \cdot \xx_t$ and $f_t(\xx_t)=\ww_t \cdot \xx_t$.
\end{lemma}
The proof is given in Appendix \ref{app-sec:lemma-1}. Now, we will find the bounds on rejection rate and mis-classification rate. 

\begin{theorem} Let $(\xx_{1}, y_{1}), \dotso , (\xx_{T}, y_{T})$ be a sequence of input instances, where $\xx_{t} \in \mathcal{X}$ and $y_{t} \in \{-1, 1\}$ and $\| \xx_{t}\| \leq R$ for all $t\in[T]$. Assume that there exists a $f(\xx)=\ww \cdot \xx$ and $\rho$ such that $\| \ww \| \leq \WW$ and $L_{dr}(y_tf(\xx_t), \rho) = 0$ for all $t\in [T]$.  
\begin{enumerate}
\item Number of examples rejected 
by DRAL (Algorithm~\ref{algo:double-ramp-active-learning}) among those for which the label was asked in this sequence is upper bounded as follows. 
\begin{equation*}
\sum \limits_{t:Q_t=1} [R_{1t} + R_{2t}] \leq \alpha^2 \|  \ww \|^2 + (1 - \alpha \rho)^2
\end{equation*}
where  $\alpha = \max \Big{(} \frac{1 + \eta^2 d^2(R^2 + 1) + 2 \eta d}{2 \eta d} ,\frac{m_{22} (1+ \eta^2 (1-d)^2(R^2 + 1) + 2 \eta (1-d) )}{2 m_{21} \eta (1+d)} \Big{)}$.
 \item Number of examples mis-classified by DRAL (Algorithm~\ref{algo:double-ramp-active-learning}) among those for which the label was asked in this sequence is upper bounded as follows.
\begin{equation*}
\sum\limits_{t:Q_t=1}M_{t} \leq {\alpha}^2{ \| \ww \| }^2 + (1 - {\alpha}\rho)^2
\end{equation*}
where  $\alpha = \max \Big{(} \frac{ \eta d(R^2 + 1) + 2}{2} ,\frac{ m_{22} ( 1+ \eta^2 (1-d)^2(R^2 + 1) + 2 \eta (1-d) )}{2 m_{21} \eta (1+d)} \Big{)}$.
\end{enumerate}
\end{theorem}
 The proof is given in Appendix \ref{app-sec-theorem-2}. The above theorem assumes that there exists $f(\xx)=\ww \cdot \xx$ and $\rho$ such that $L_{dr}(y_tf(\xx_t), {\rho})=0$ for all $t \in [T]$. This means that the data is linearly separable. In such a case, the number of mistakes made by the algorithm on unrejected examples as well as the number of rejected examples are upper bounded by a complexity term and are independent of $T$. Now, we derive the bounds when the assumption $L_{dr}(y_tf(\xx_t), {\rho})=0,\;t\in[T]$ does not hold for any $f(\xx)=\ww \cdot \xx$ and $\rho$.

\begin{theorem}
Let $(\xx_{1}, y_{1}), (\xx_{2}, y_{2}), \dotso , (\xx_{T}, y_{T})$ be a sequence of input instances, where $\xx_{t} \in \mathcal{X}$ and $y_{t} \in \{-1, 1\}$ and $\| \xx_{t}\| \leq R$ for all $t\in [T]$. Then, for any given $f(\xx)=\ww \cdot \xx$ ($\| \ww \| \leq \WW$) and $\rho$, we observe the following. 
\begin{enumerate}
    \item 
Number of rejected examples by DRAL (Algorithm~\ref{algo:double-ramp-active-learning}) among those for which the label was asked in this sequence is upper bounded as follows.  
\begin{equation*}
\sum \limits_{t:Q_t=1} [R_{1t} + R_{2t}] \leq \alpha^2 \|  \ww \|^2 + (1 - \alpha \rho)^2 +  \sum_{t=1}^T \frac{ 2\eta \alpha }{ m }  L_{dr}(y_tf(\xx_{t}), {\rho})
\end{equation*}
where  $\alpha = \max \begin{cases} \frac{1 + \eta^2 d^2(R^2 + 1) + 2 \eta d}{2 \eta d} \\ \frac{m_{22}(1+ \eta^2 (1-d)^2(R^2 + 1) + 2 \eta (1-d))}{2 m_{21} \eta (1+d)} \end{cases}$.
    \item 
The number of misclassified examples by DRAL (Algorithm~\ref{algo:double-ramp-active-learning}) is upper bounded as follows. 
\begin{align*}
    \sum \limits_{t:Q_t=1} M_{t} \leq \alpha^2 \| \ww \|^2 + (1 - \alpha \rho)^2 +  \sum_{t=1}^T \frac{ 2\eta \alpha }{ m } L_{dr}(y_tf(\xx_{t}), {\rho})
\end{align*}
where $\alpha = \max \begin{cases} \frac{ \eta d(R^2 + 1) + 2}{2} \\ \frac{m_{22}(1+ \eta^2 (1-d)^2(R^2 + 1) + 2 \eta (1-d))}{2 m_{21} \eta (1+d)} \end{cases}$.
\end{enumerate}
\end{theorem}
 The proof is given in Appenxid \ref{app-sec:theorem-3}. We see that when the data is not linearly separable, the number of mistakes made by the algorithm is upper bounded by the sum of complexity term and sum of the losses using a fixed classifier.

\section{Active Learning Using Double Sigmoid Loss Function}
\label{section:ds_active} 
We observe that double ramp loss is not smooth. Moreover, $L_{dr}$ is constant whenever $yf(\xx) \in [\rho+1,\infty)\cup (-\infty, -\rho-1] \cup [-\rho+1 , \rho-1]$. Thus, when loss $L_{dr}$ for an example $\xx$ falls in any of these three regions, the gradient of the loss becomes zero. The zero gradient causes no update. Thus, there is no benefit of asking the labels when an example falls in one of these regions. However, we don't want to ignore these regions completely. To capture the information in these regions, we need to change the loss function in such a way that the gradient does not vanish completely in these regions. To ensure that, we propose a new loss function.

\subsection{Double Sigmoid Loss}
We propose a new loss function for reject option classification by combining two sigmoids as follows. We call it {\em double sigmoid loss} function $L_{ds}$.
\begin{align*}
    L_{ds}(yf(\xx), \rho) = 2d \sigma (yf(\xx) - \rho) + 2(1-d) \sigma (y f(\xx) + \rho)
\end{align*}
where $\sigma(a) = \left(1+e^{ \gamma a}\right)^{-1}$ is the sigmoid function ($\gamma >0$).
Figure~\ref{fig:DSL} shows the double sigmoid loss function. $L_{ds}$ is a smooth non-convex surrogate of loss $L_d$ (see eq.(\ref{eq:0-d-1})). We also see that for the double sigmoid loss, the gradient in the regions $yf(\xx) \in [\rho+1,\infty)\cup (-\infty, -\rho-1] \cup [-\rho+1 , \rho-1]$ does not vanish unlike double ranp loss. 
\begin{figure}[h]
\begin{center}
    \includegraphics[width=0.7\columnwidth]{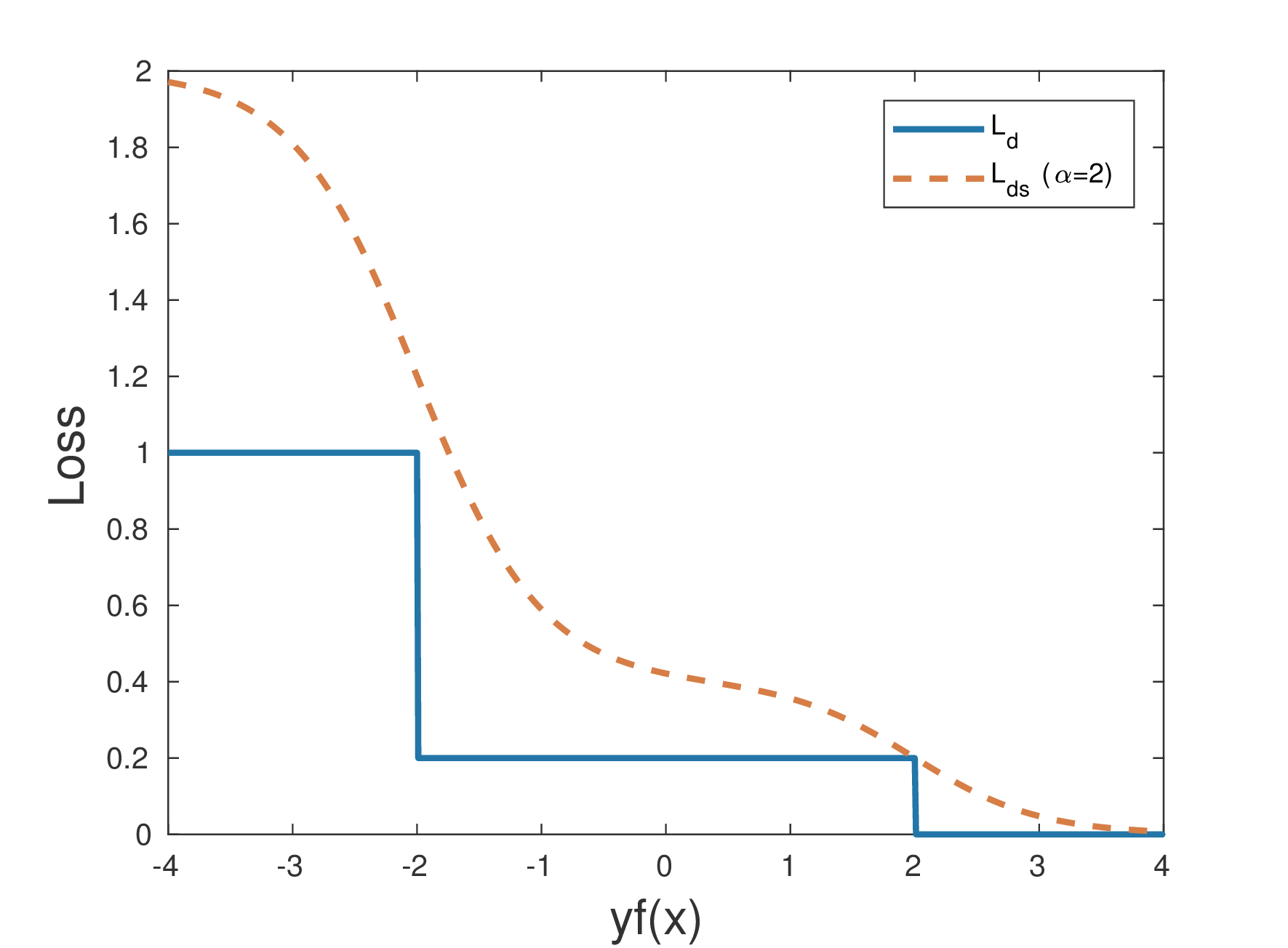}
    \caption{Double sigmoid loss with $\gamma = 2$.}
    \label{fig:DSL}
    \end{center}
\end{figure}
Below we establish that the loss $L_{ds}$ is $\beta$-smooth.\footnote{A function $f$ is $\beta$-smooth if for all $x, y \in$ Domain($f$),
\begin{equation*}
    \| \nabla f(x) - \nabla f(y)  \| \leq \beta \| x - y  \|.
\end{equation*}
} 
\begin{lemma}
\label{lemma:smoothness}
Assuming $\| \xx \| \leq \RR$,  Double sigmoid loss $L_{ds}(yf(\xx), \rho)$ is $\beta-$smooth with constant $\beta =  \frac{\gamma^2}{5} \big[ \RR^2 + 1 \big] $. 
\end{lemma}
The proof is given in Appendix \ref{app-sec:lemma-4}.

\subsection{Query Probability Function}
In the case of DRAL, we saw that the gradient of $L_{dr}$ becomes nonzero only in the region $yf(\xx) \in [\rho-1,\rho +1]$. So, we ask the labels only when examples fall in this region. However, in case of double sigmoid loss, the gradient does not vanish. Thus, to perform active learning using $L_{ds}$, we need to ask the labels selectively. 

We propose a query probability function to set the label query probability at trial $t$. The query probability function should carry the following properties. In the loss $L_d$ (see eq.(\ref{eq:0-d-1})), we see two transitions. One at $yf(\xx)=\rho$ (transition between correct classification and rejection) and another at $yf(\xx)=-\rho$ (transition between rejection and misclassification). Any example falling closer to one of these transitions captures more information about the two transitions. We want the query probability function to be such that it gives higher probabilities near these transitions. 
Examples that are correctly classified with a good margin, examples misclassified with a considerable margin, and examples in the middle of the reject region do not carry much information. Such examples are also situated away from the transition regions. Thus, query probability should decrease as we move away from these decision boundaries. Therefore, we ask the label in these regions with less probability.
Considering these desirable properties, we propose the following query probability function.  
\begin{align}
\label{eq:proper-bimodal-probability}
     p_{t} = 4 \; \sigma ( | f_{t}( \xx_{t} ) | - \rho_{t} )\left( 1 - \sigma( | f_{t} (\xx_{t}) |  - \rho_{t}) \right)
\end{align}
where $f_{t}(\xx_t) = \ww_{t} \cdot \xx_t$. Figure~\ref{fig:Query-probability} shows the graph of the query probability function. 
We see that the probability function has two peaks. One peak is at $yf(\xx)=\rho$ (transition between correct classification and rejection) and another at $yf(\xx)=-\rho$ (transition between rejection and misclassification).  
\begin{figure}[h]
\begin{center}
    \includegraphics[width=0.65\columnwidth]{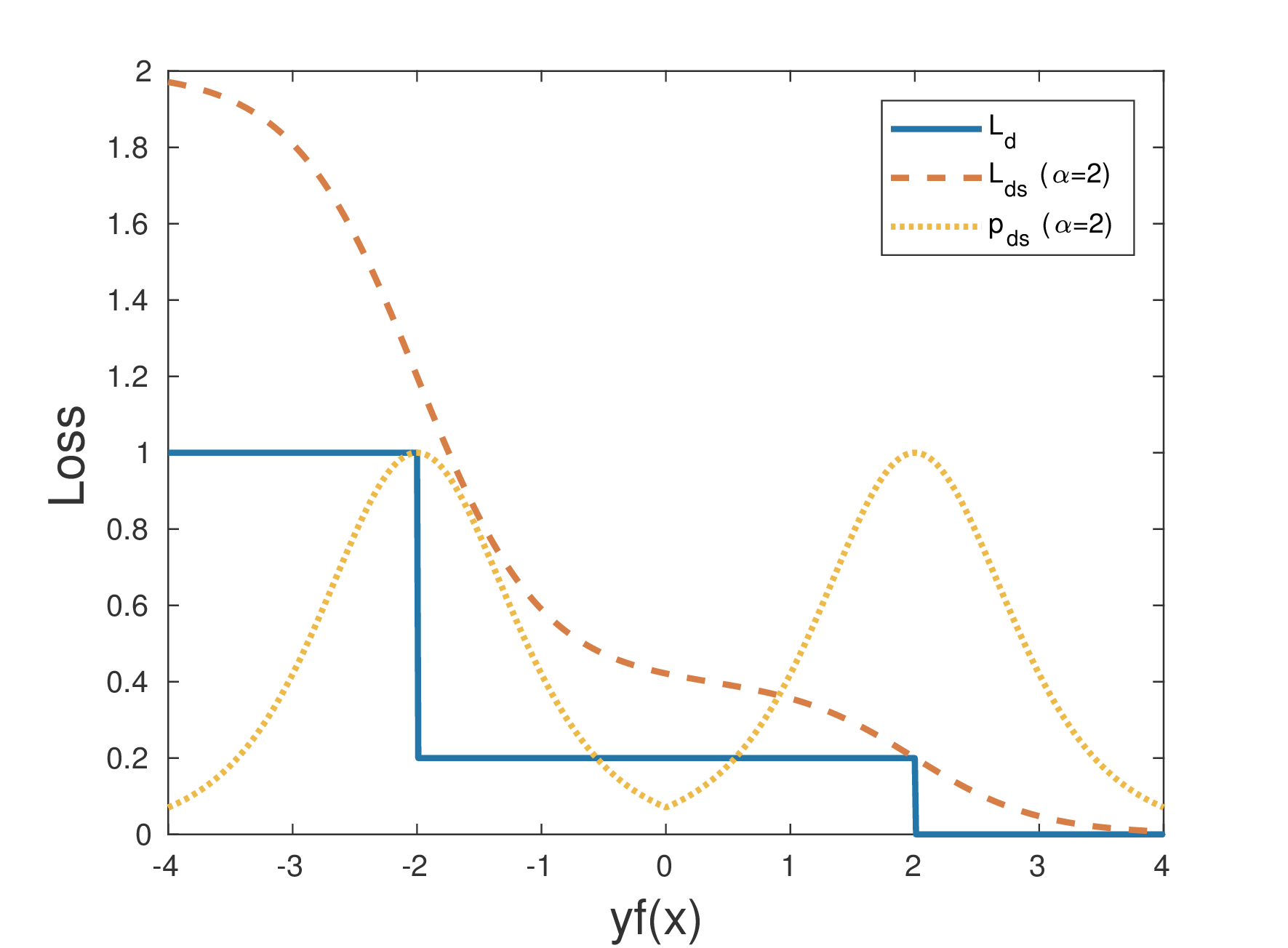}
    \caption{Query Probability Function}
    \label{fig:Query-probability}
    \end{center}
\end{figure}

\subsection{Double Sigmoid Based Parameter Updates}
The parameter update equations using $L_{ds}$ is as follows.
\begin{align}
    \nonumber & \ww_{t+1}=\ww_t -\eta \nabla_{\ww_t} L_{ds}(y_tf(\xx_t), \rho_t)\\
    \nonumber &= \ww_t-2y_{t} \alpha \xx_{t} \Big{[}d\sigma ( y_{t} f_{t}( \xx_{t} ) - \rho_{t} )\left( 1 - \sigma( y_{t} f_{t}( \xx_{t} )  - \rho_{t}) \right)\\
    &+ (1 - d)\sigma ( y_{t} f_{t}( \xx_{t} ) + \rho_{t} )\left( 1 - \sigma( y_{t} f_{t}( \xx_{t} )  + \rho_{t}) \right)  \Big{]} \label{eq:proper-bimodal-w-update}
    \\
    \nonumber & \rho_{t+1}=\rho_t - \eta \nabla_{\rho_t} L_{ds}(y_tf(\xx_t), \rho_t)\\
    \nonumber &= \rho_t+ 2\alpha \Big{[}d  \sigma ( y_{t} f_{t}( \xx_{t} ) - \rho_{t} )\left( 1 - \sigma( y_{t} f_{t}( \xx_{t} )  - \rho_{t}) \right)\\
    &\;\;\;\;\;- (1 - d)  \sigma ( y_{t} f_{t}( \xx_{t} ) + \rho_{t} )\left( 1 - \sigma( y_{t} f_{t}( \xx_{t} ) + \rho_{t}) \right)  \Big{]} \label{eq:proper-bimodal-rho-update}
\end{align}

\noindent Now, we will explain the update equations for $\ww$ and $\rho$. 
\begin{enumerate}
    \item When an example is correctly classified with good margin (i.e. $y_{t}f_{t}(\xx_{t}) >> 0$) then the active learning algorithm will update $\ww$ by a small factor of $y_{t} \xx_{t}$ and will reduce the rejection width $(\rho)$ because for $y_{t}f_{t}(\xx_{t}) >> 0$, $ d \sigma ( y_{t} f_{t}( \xx_{t} ) - \rho_{t} )\left( 1 - \sigma( y_{t} f_{t}( \xx_{t} )  - \rho_{t}) \right) > (1-d) \sigma ( y_{t} f_{t}( \xx_{t} ) + \rho_{t} )\left( 1 - \sigma( y_{t} f_{t}( \xx_{t} )  + \rho_{t}) \right)$.
    \item When an example is misclassified with good margin (i.e. $y_{t}f_{t}(\xx_{t}) << 0$) then the active learning algorithm will update $\ww$ by a large factor of $y_{t} \xx_{t}$ and will increase the rejection width $(\rho)$ because for $y_{t}f_{t}(\xx_{t}) << 0$, $d \sigma ( y_{t} f_{t}( \xx_{t} ) - \rho_{t} )\left( 1 - \sigma( y_{t} f_{t}( \xx_{t} )  - \rho_{t}) \right) < (1-d) \sigma ( y_{t} f_{t}( \xx_{t} ) + \rho_{t} )\left( 1 - \sigma( y_{t} f_{t}( \xx_{t} )  + \rho_{t}) \right)$.
    \end{enumerate}
    We use the acronym DSAL for double sigmoid based active learning. DSAL is described in Algorithm~\ref{algo:double-sigmoid-active-learning}.

\begin{algorithm}
\caption{{\bf D}ouble {\bf S}igmoid Loss {\bf A}ctive {\bf L}earning (DSAL)}
\label{algo:double-sigmoid-active-learning}
\begin{algorithmic}
\State {\bf Input:} $d \in (0, 0.5)$, step size $\eta$
\State {\bf Output:} Weight vector $\ww$, Rejection width $\rho$.
\State {\bf Initialize:} $\ww_{1}, \rho_{1}$
\For{$t = 1,..,T$} 
\State Sample $\xx_{t} \in \mathbb{R}^{d}$
\State Set $f_{t}(\xx_t)=\ww_{t} \cdot \xx_t$ 
\State Set $p_{t}=4 \sigma ( | f_{t}( \xx_{t} ) | - \rho_{t} )\left( 1 - \sigma( | f_{t} (\xx_{t}) |  - \rho_{t}) \right)$ 
\State Randomly sample $z_{t} \in \{0, 1\}$ from Bernoulli($p_{t}$). 
\If{$z_{t} == 1$}
\State Query the label $y_{t}$ of $\xx_t$.
\State Find $\ww_{t+1}$ using eq.(\ref{eq:proper-bimodal-w-update}).
\State Find $\rho_{t+1}$ using eq.(\ref{eq:proper-bimodal-rho-update}).
\Else{}
\State $\ww_{t+1} = \ww_{t}$.
\State $\rho_{t+1} = \rho_{t}$.
\EndIf
\EndFor
\end{algorithmic}
\end{algorithm}

\subsection{Convergence of DSAL}
 In the case of DRAL, the mistake bound analysis was possible as $L_{dr}$ increases linearly in the regions where its gradient is nonzero. However, we don't see similar behavior in double sigmoid loss $L_{ds}$. Thus, we are not able to carry out the same analysis here. Instead, we here show the convergence of DSAL to local minima. For which, we borrow the techniques from online non-convex optimization. In online non-convex optimization, it is challenging to converge towards a global minimizer. It is a common practice to state the convergence guarantee of an online non-convex optimization algorithm by showing it's convergence towards an $\epsilon$-approximate stationary point. In our case, it means that for some $t$, $\| \nabla L_{ds} (y_{t} f_{t} (\xx_{t}), \rho_t) \|^2 \leq \epsilon$. To prove the convergence of DSAL, we use the notion of local regret defined in \citep{DBLP:journals/corr/abs-1708-00075} . 
\begin{definition}
The local regret for an online algorithm is
\begin{equation*}
\mathcal{R} ( T ) = \sum_{t=1}^T \| \nabla L_{ds} (y_{t}f_{t}(\xx_{t}) , \rho_{t}) \|^2. 
\end{equation*}
where $T$ is the total number of trials. (Defined in \citep{DBLP:journals/corr/abs-1708-00075})
\end{definition}
Thus, in each trial, we incur a regret, which is the squared norm of the gradient of the loss. When we reach a stationary point, the gradient will vanish and hence the norm. Note that the convergence here requires that the objective function should be $\beta$-smooth. In this case, $L_{ds}$ holds that property, as shown in Lemma~\ref{lemma:smoothness}. Thus, we can use the convergence approach proposed in \citep{DBLP:journals/corr/abs-1708-00075}.\footnote{$L_{dr}$ does not have sufficient smoothness properties required in \citep{DBLP:journals/corr/abs-1708-00075}. Thus, we do not present these convergence results for DRAL.}

\begin{theorem}
\label{theorem:local-regret-bound}
If we choose $\eta = \frac{5}{ \gamma^2 \big[ \RR^2 + 1 \big] }$, then using smoothness condition of $L_{ds}(yf(\xx), \rho)$, the local regret of DSAL algorithm is bounded as follows. 
\begin{equation*}
    \mathcal{R}(T) \leq \frac{4 \gamma^2}{5}  \left( \RR^2 + 1 \right) \left( T + 1 \right)
\end{equation*}
\end{theorem}
The proof is given in Appendix \ref{app-sec:theorem-6}. To prove that DSAL reaches $\epsilon-$stationary point in expectation over iterates, we use following result of \citep{DBLP:journals/corr/abs-1708-00075}. 
\begin{align}
\label{eq:expected-gradient-bound}
    \mathop{\mathbb{E}}_{t \sim \text{Unif} [T] }\left[ \|  \nabla L_{ds}(y f(\xx), \rho) \|^2  \right] \leq \frac{\mathcal{R}(T)}{T} 
\end{align}

\begin{figure*}[t]
\begin{center}
\begin{tabular}{ccccc}
\includegraphics[width=0.38\columnwidth]{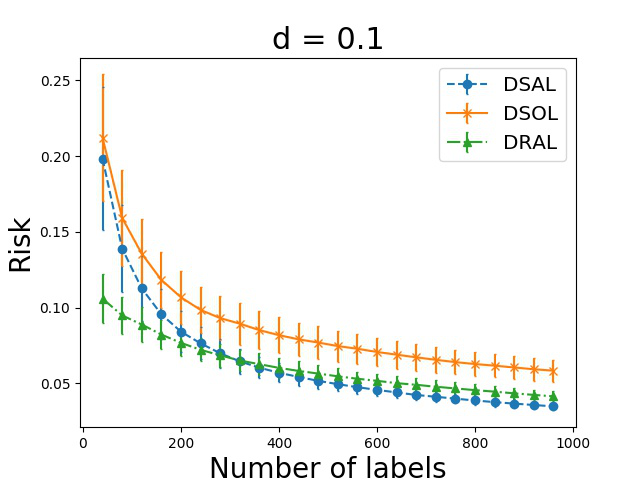}&
\includegraphics[width=0.38\columnwidth]{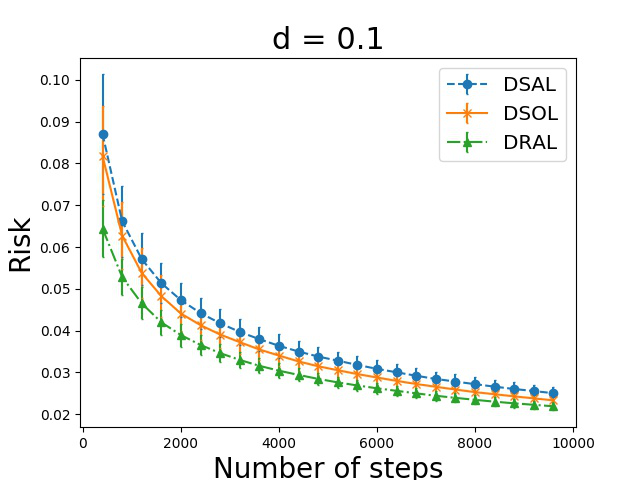}&
\includegraphics[width=0.38\columnwidth]{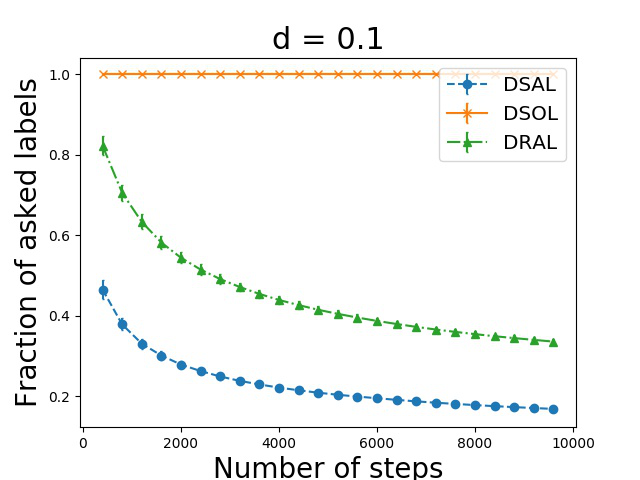}&
\includegraphics[width=0.38\columnwidth]{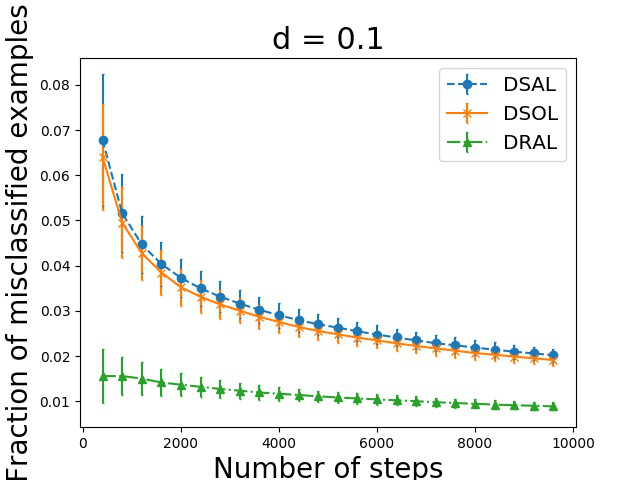}&
\includegraphics[width=0.38\columnwidth]{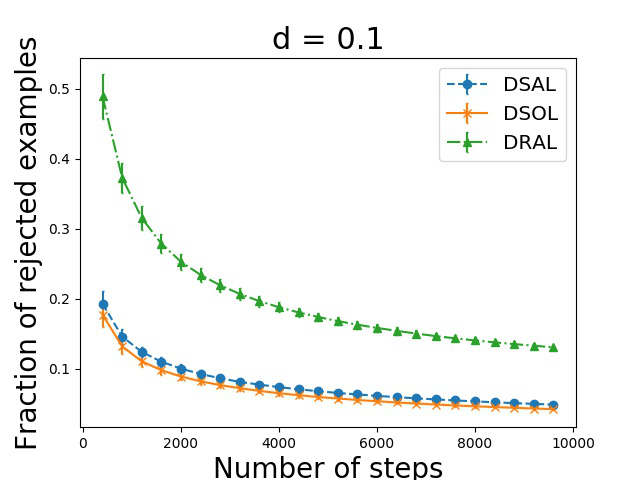}\\
\includegraphics[width=0.38\columnwidth]{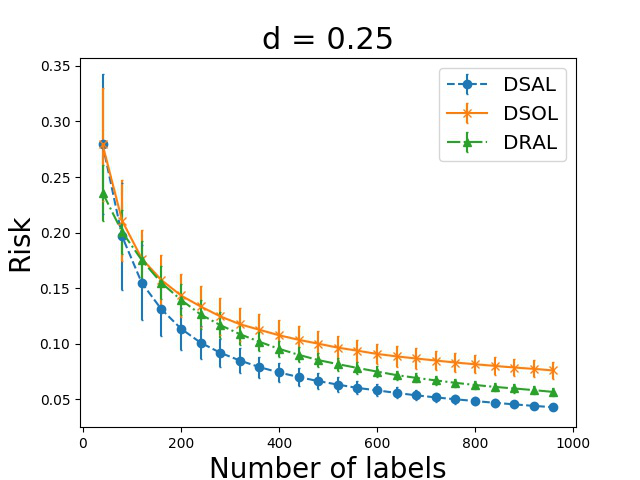}&
\includegraphics[width=0.38\columnwidth]{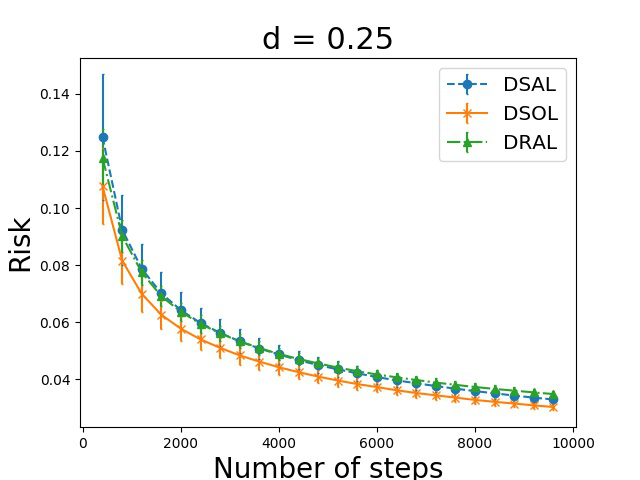}&
\includegraphics[width=0.38\columnwidth]{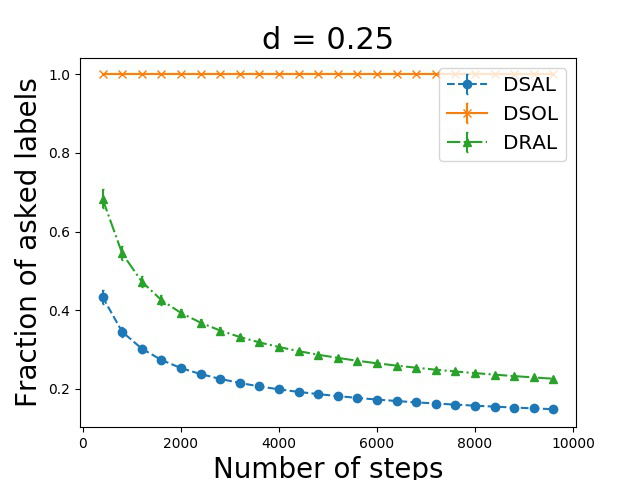}&
\includegraphics[width=0.38\columnwidth]{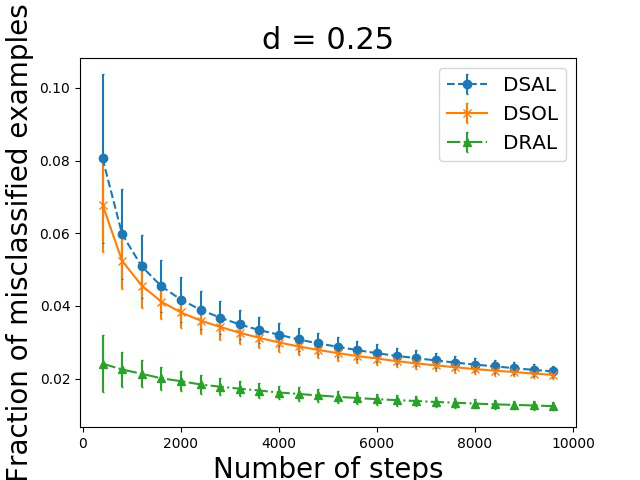}&
\includegraphics[width=0.38\columnwidth]{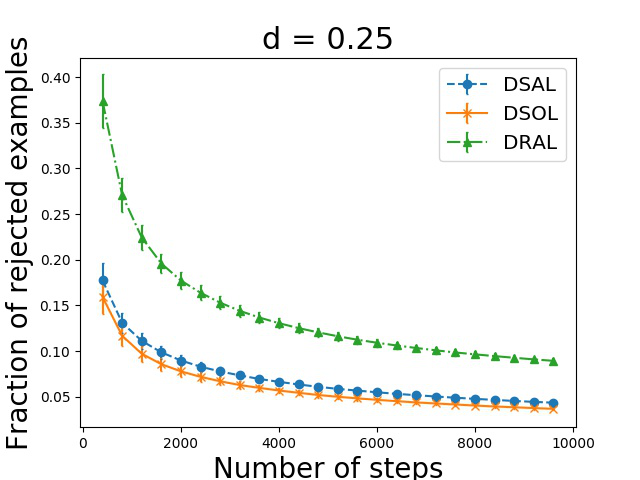}\\
\includegraphics[width=0.38\columnwidth]{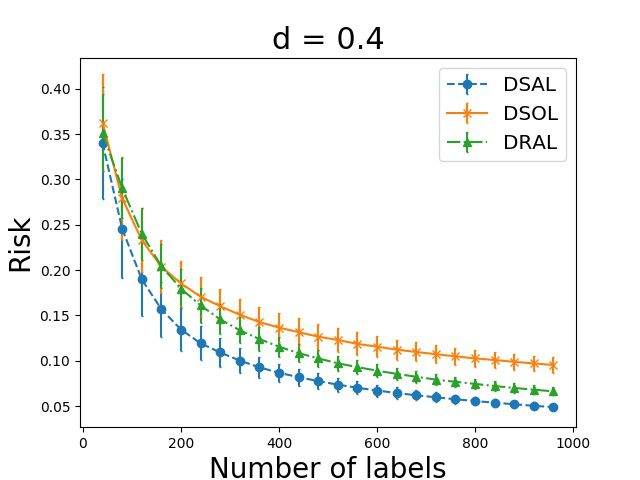}&
 \includegraphics[width=0.38\columnwidth]{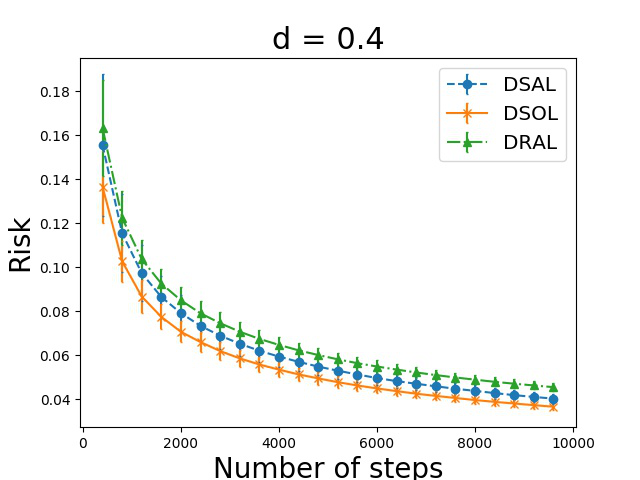}&
\includegraphics[width=0.38\columnwidth]{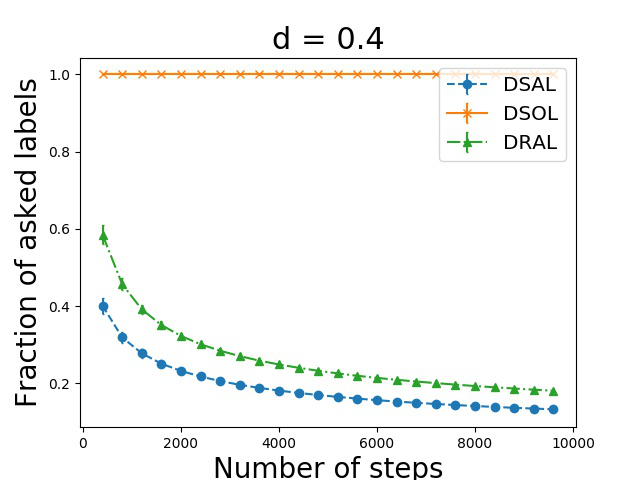} &
\includegraphics[width=0.38\columnwidth]{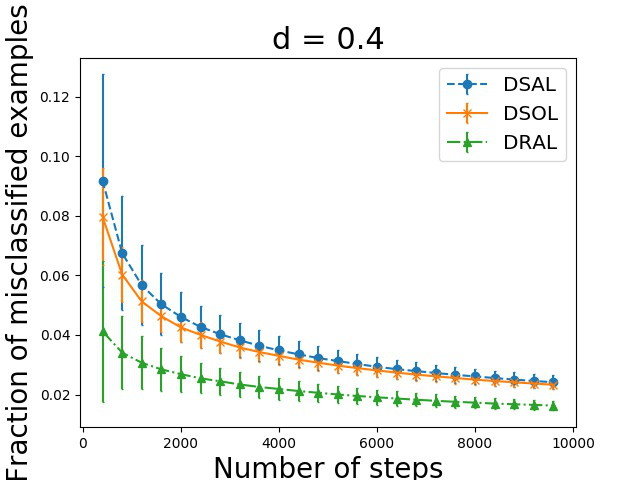} &
\includegraphics[width=0.38\columnwidth]{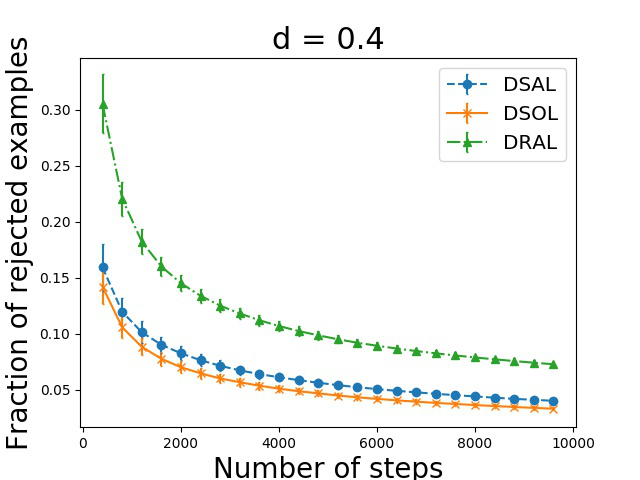} 
\end{tabular}
\end{center}
\caption{Comparison plots for Gisette dataset with linear Kernel function. }
\label{fig:gisette-results}
\end{figure*}

\begin{corollary}
\label{corollary-thm}
For DSAL algorithm, 
\begin{align}
\label{eq:corollary}
    \mathop{\mathbb{E}}_{  t \sim \text{Unif} [T] }\left[ \|  \nabla L_{ds}(y f(\xx),\rho) \|^2  \right] \leq 
    \frac{4 \gamma^2}{5} \left( \RR^2 + 1 \right) \left( 1 + \frac{1}{T} \right)
\end{align}
\end{corollary}

Using theorem \ref{theorem:local-regret-bound} and eq. (\ref{eq:expected-gradient-bound}), we can get the required result of the Corollary. In the Corollary, We see that upper bound on the expectation of the square of the gradient is inversely proportional to $T$; hence, decreases as the total number of trials $T$ increases. It means that the probability of DSAL algorithm reaches to $\epsilon-$stationary point increases as $T$ increases. 

\section{Experiments}
\label{section:exp}
We show the effectiveness of the proposed active learning approaches on Gisette, Phishing and Guide datasets available on UCI ML repository \citep{Lichman:2013}. 

\subsection{Experimental Setup}
We evaluate the performance of our approaches to learning linear classifiers. In all our simulations, we initialize step size by a small value, and after every trial, step size decreases by a small constant. Parameter $\alpha$ in the double sigmoid loss function is chosen to minimize the average risk and average fraction of queried labels (averaged over 100 runs). 

We need to show that the proposed active learning algorithms are effectively reducing the number of labeled examples required while achieving the same accuracy as online learning. Thus, we compare the active learning approaches with an online algorithm that updates the parameters using gradient descent on the double sigmoid loss at every trial. We call this online algorithm as DSOL (double sigmoid loss based online learning).

\subsection{Simulation Results}
We report the results for three different values of $d \in \{ 0.1, 0.25, 0.4 \}$. The results provided here are based on 100 repetitions of a total number of trial ($T$) equal to 10000. For every value of $d$, we find the average of risk, the fraction of asked labels, fraction of misclassified examples, and fraction of rejected examples over 100 repetitions. We plotted the average of each quantity (e.g., risk, the fraction of asked labels, etc.) as a function of $t \in [T]$. Moreover, the standard deviation of the quantity is denoted by error bar in figures. Figure~\ref{fig:gisette-results}, \ref{fig:phishing-results} and \ref{fig:guidesw-results} show experimental results for Gisette and Phishing  and Guide datasets. We observe the following. 

\begin{figure*}[t!]
\begin{center}
\begin{tabular}{ccccc}
\includegraphics[width=0.38\columnwidth]{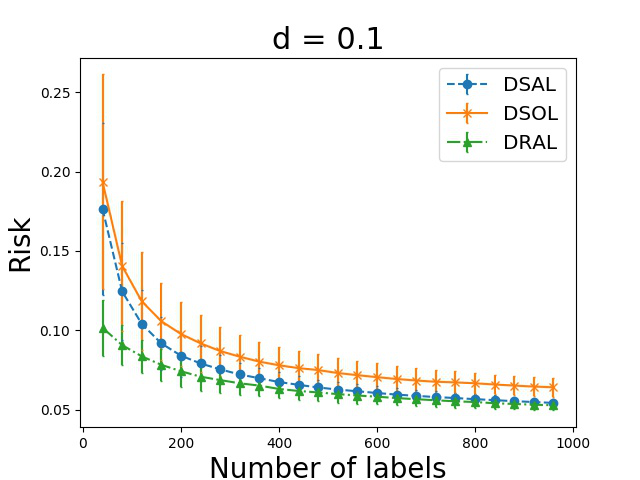}&
\includegraphics[width=0.38\columnwidth]{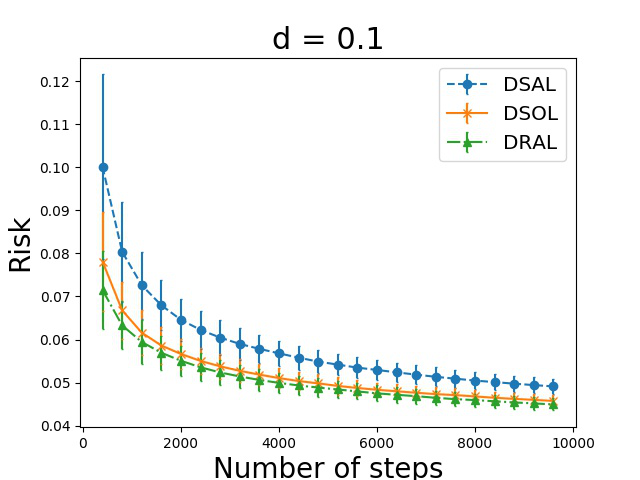}&
\includegraphics[width=0.38\columnwidth]{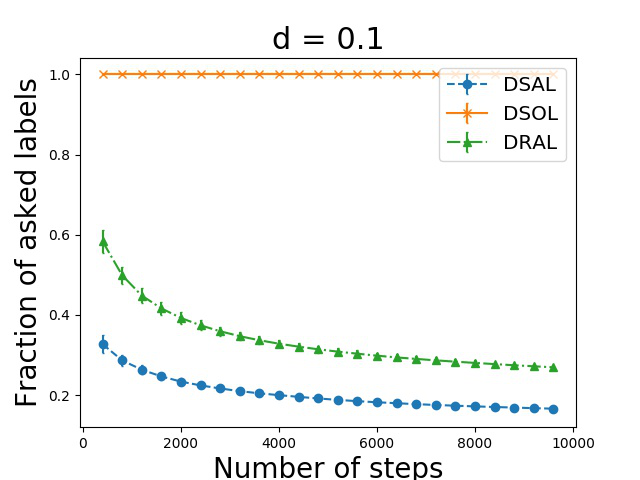}&
\includegraphics[width=0.38\columnwidth]{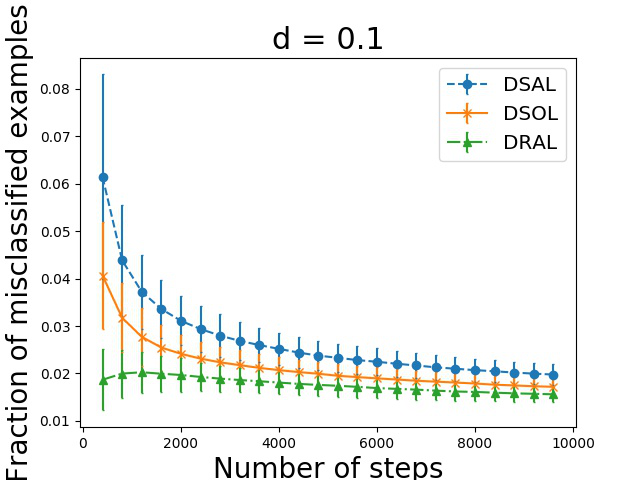}&
\includegraphics[width=0.38\columnwidth]{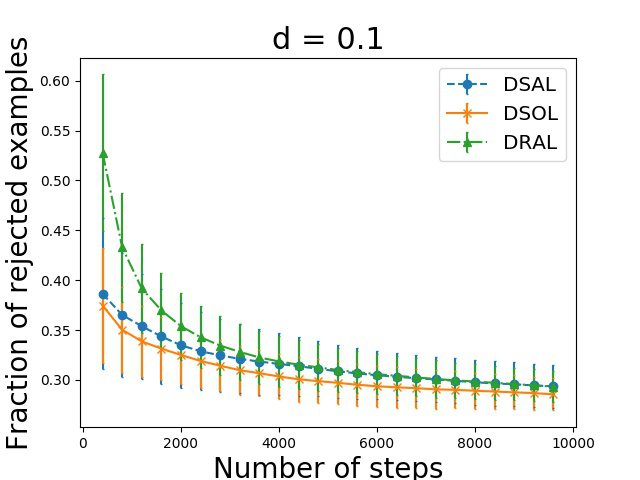}\\
\includegraphics[width=0.38\columnwidth]{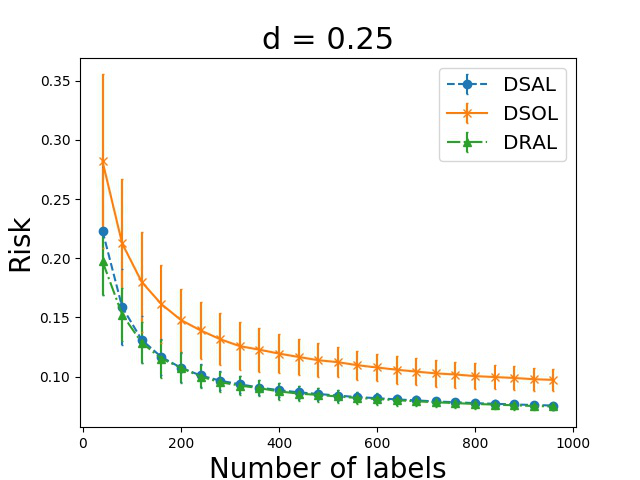}&
\includegraphics[width=0.38\columnwidth]{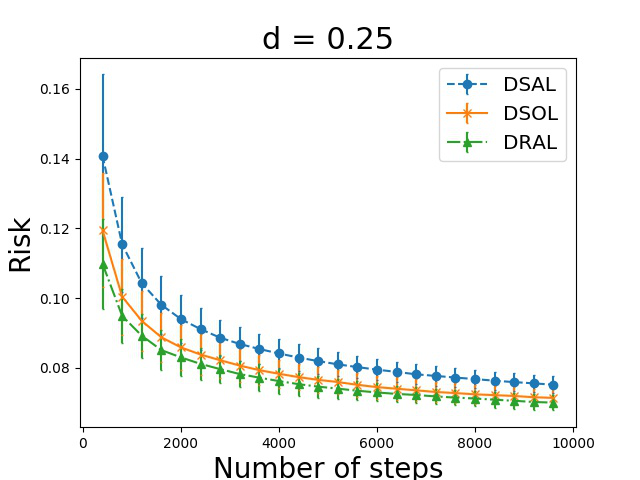}&
\includegraphics[width=0.38\columnwidth]{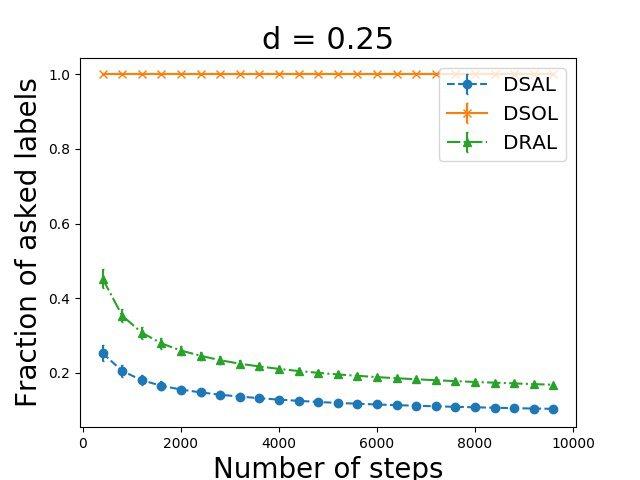}&
\includegraphics[width=0.38\columnwidth]{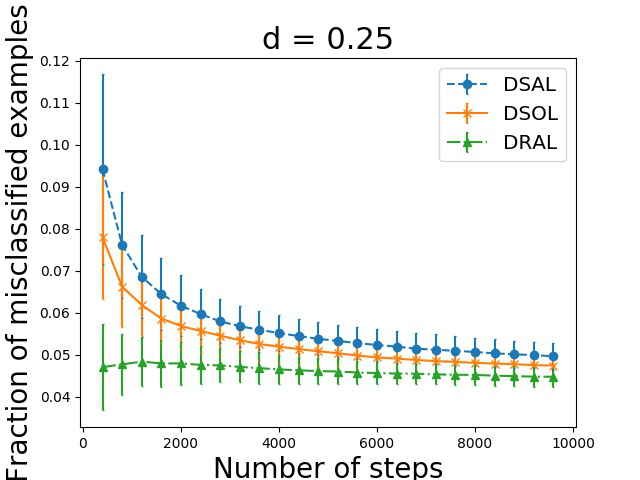}&
\includegraphics[width=0.38\columnwidth]{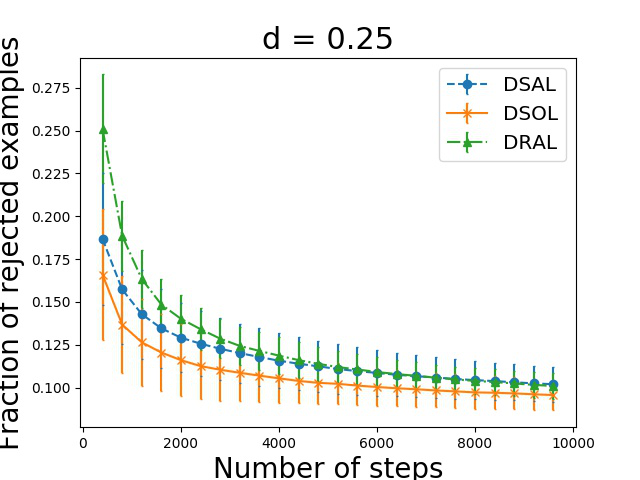}\\
\includegraphics[width=0.38\columnwidth]{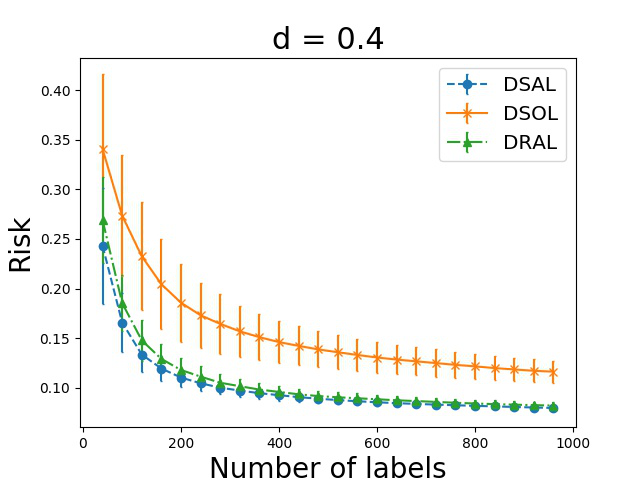}&
\includegraphics[width=0.38\columnwidth]{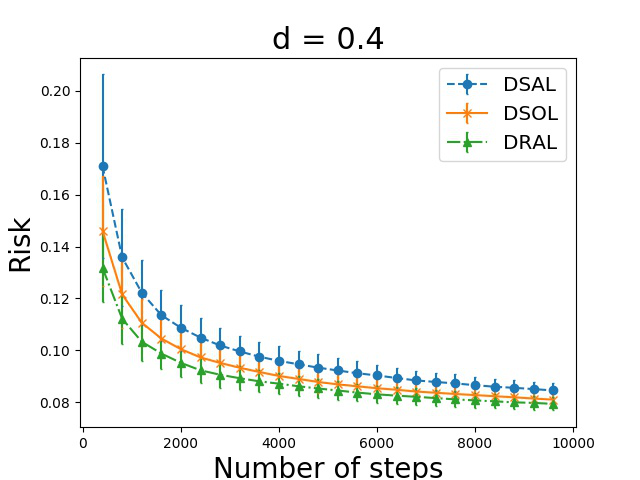}&
\includegraphics[width=0.38\columnwidth]{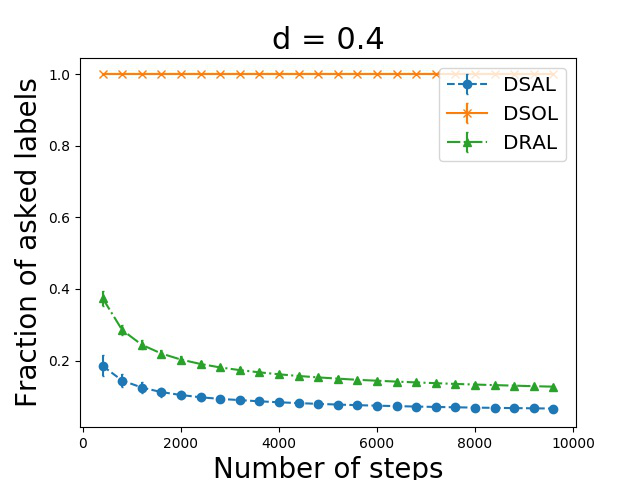}&
\includegraphics[width=0.38\columnwidth]{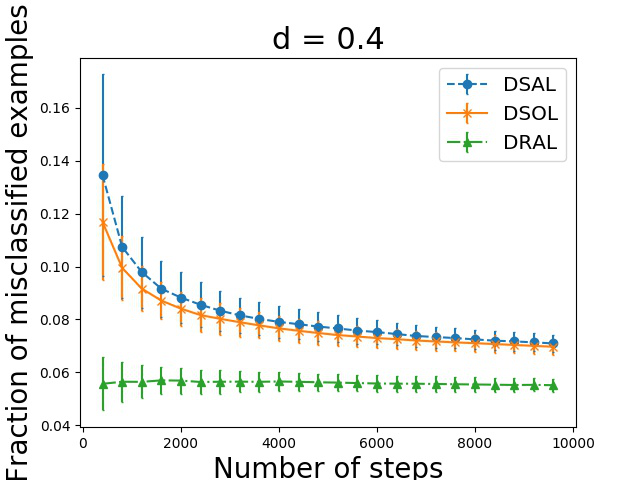}&
\includegraphics[width=0.38\columnwidth]{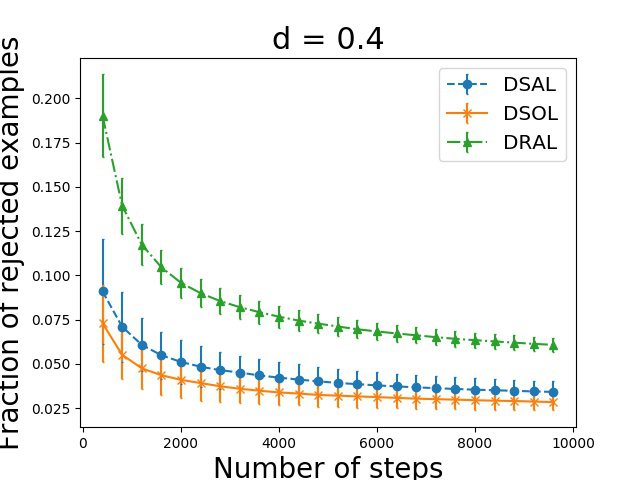}
\end{tabular}
\end{center}
\caption{Comparison plots for Phishing dataset with linear Kernel function. }
\label{fig:phishing-results}
\end{figure*}

\begin{figure*}[ht!]
\begin{center}
\begin{tabular}{ccccc}
\includegraphics[width=0.38\columnwidth]{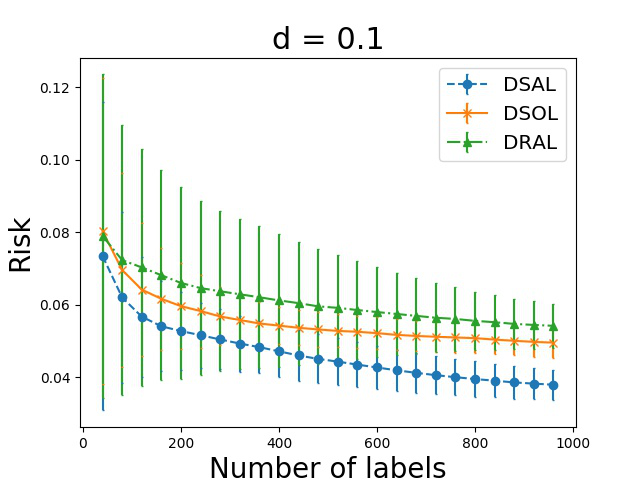}&
\includegraphics[width=0.38\columnwidth]{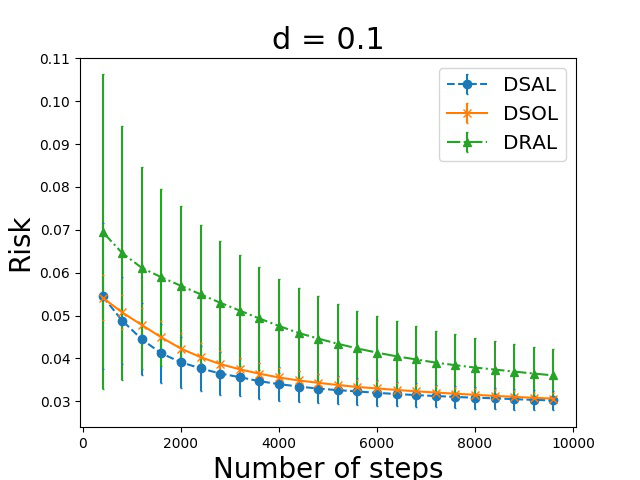}&
\includegraphics[width=0.38\columnwidth]{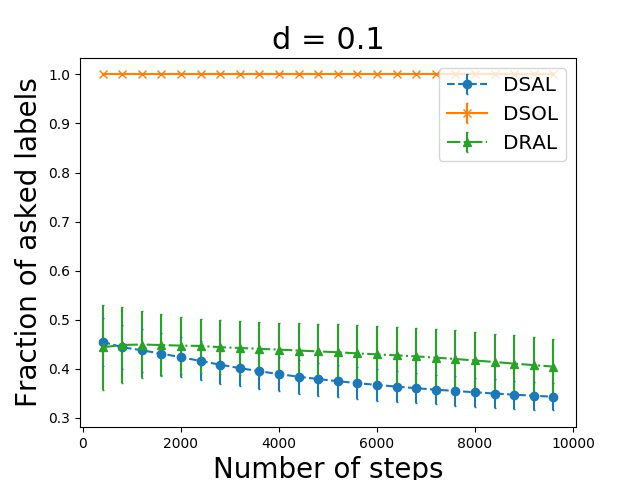}&
\includegraphics[width=0.38\columnwidth]{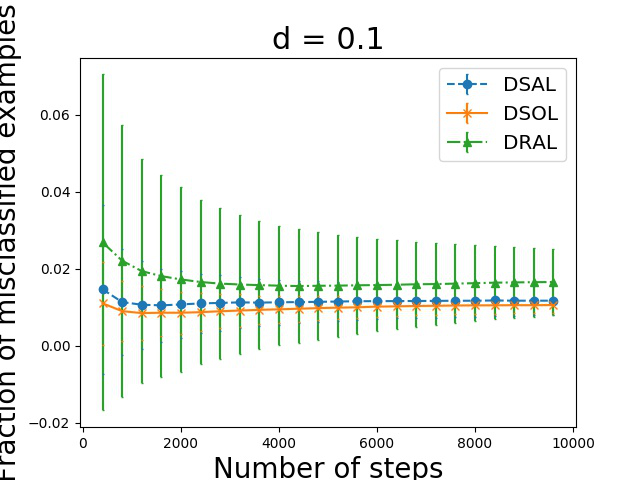}&
\includegraphics[width=0.38\columnwidth]{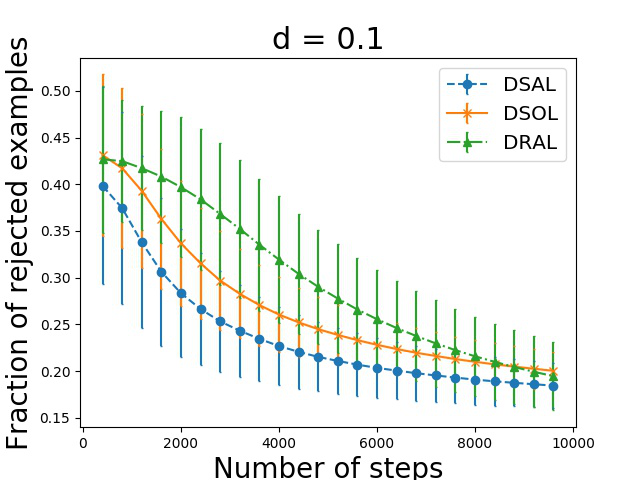}\\
\includegraphics[width=0.38\columnwidth]{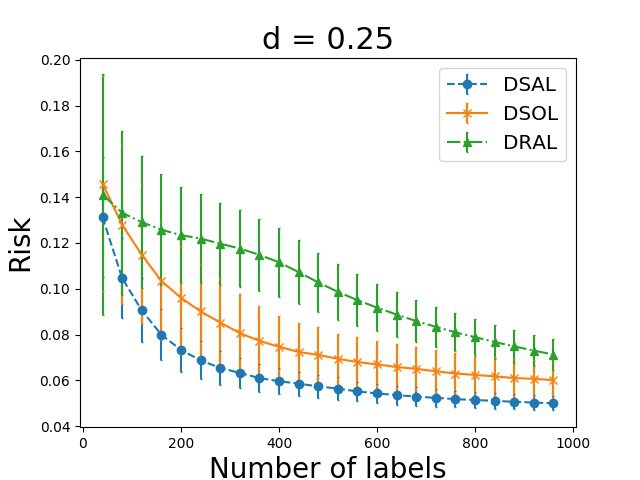}&
\includegraphics[width=0.38\columnwidth]{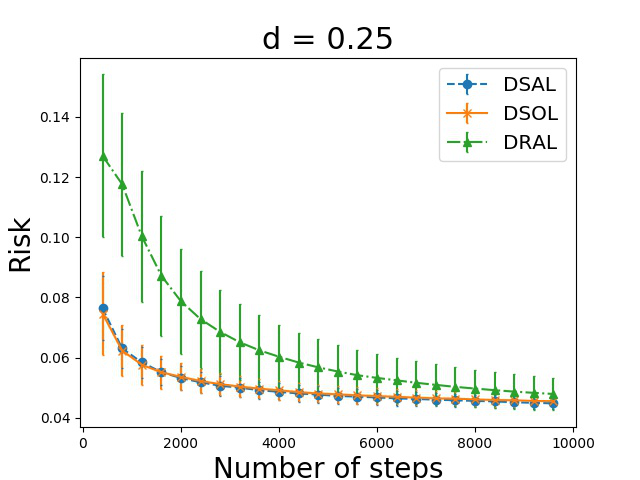}&
\includegraphics[width=0.38\columnwidth]{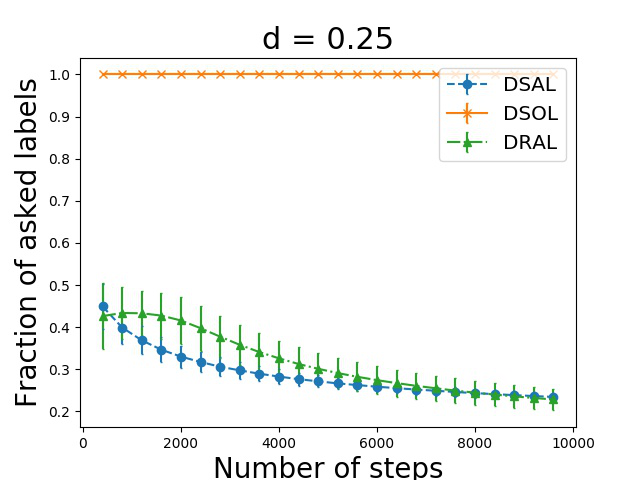}&
\includegraphics[width=0.38\columnwidth]{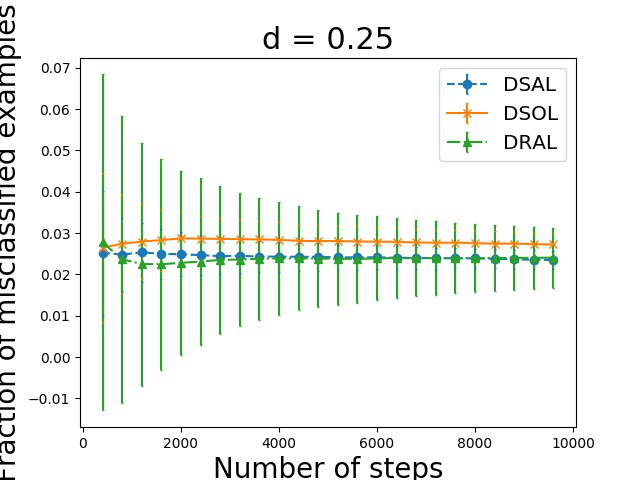}&
\includegraphics[width=0.38\columnwidth]{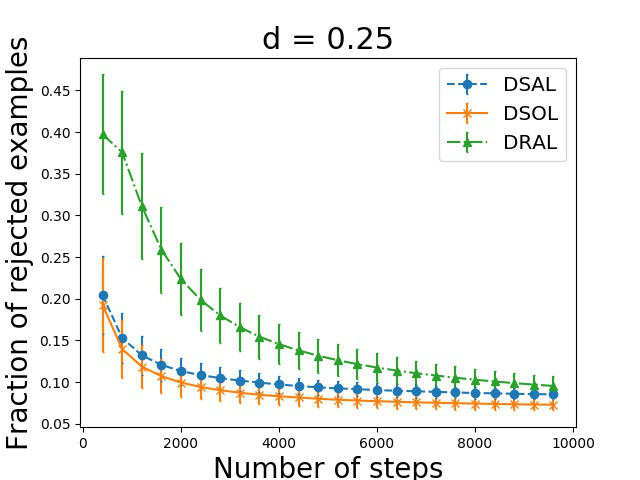}\\
\includegraphics[width=0.38\columnwidth]{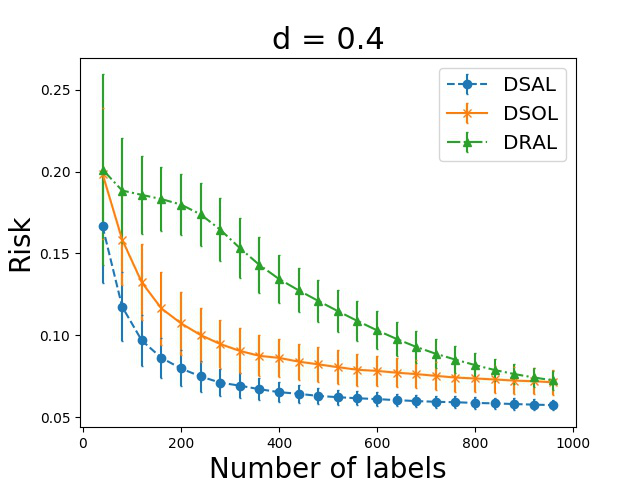}&
\includegraphics[width=0.38\columnwidth]{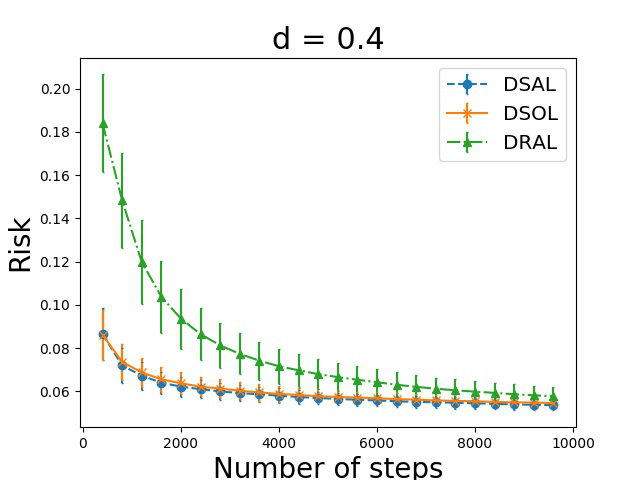} &
\includegraphics[width=0.38\columnwidth]{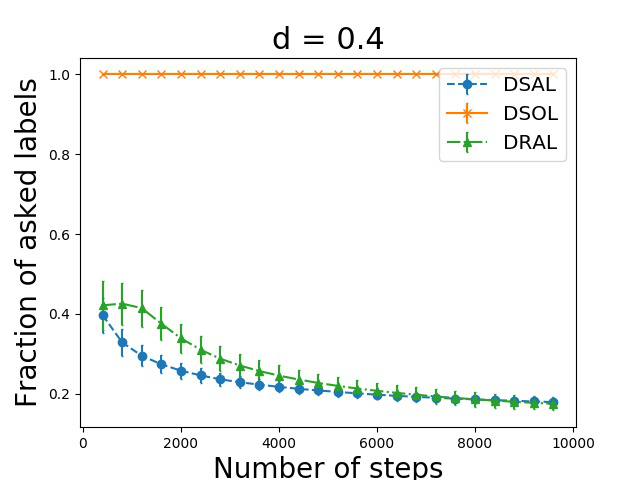} &
\includegraphics[width=0.38\columnwidth]{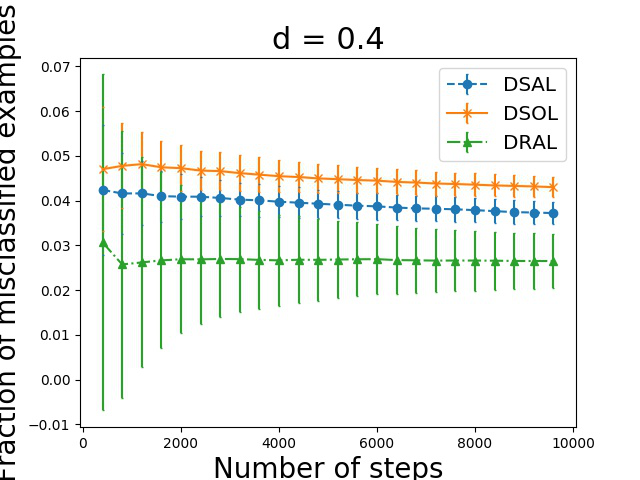} &
\includegraphics[width=0.38\columnwidth]{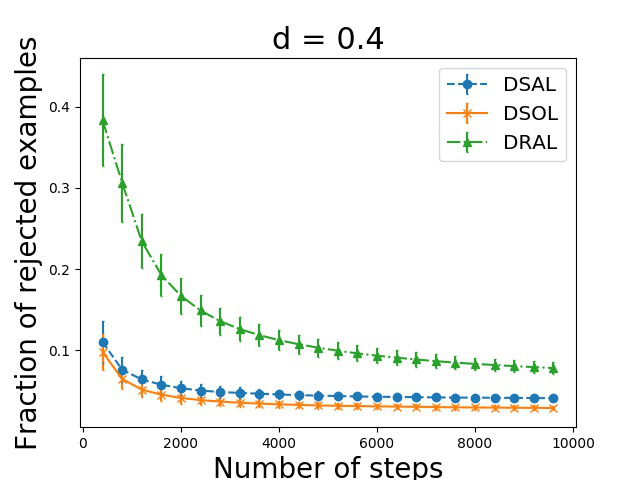} 
\end{tabular}
\end{center}
\caption{Comparison plots for Guide dataset with polynomial kernel function. }
\label{fig:guidesw-results}
\end{figure*}

\begin{itemize}
   \item {\bf Label Complexity Versus Risk:} The first column in each figure shows how the risk goes down with the number of asked labels. For Gisette and Phishing datasets, given the number of queried labels, both DSAL and DRAL achieve lower risk compared to DSOL. For Guide dataset, DSAL always makes lower risk compared to DSOL for a given number of queried labels.
   For Gisette and Guide datasets, DSAL achieves lower risk compared to DRAL with the same number of label queries. For Phishing dataset, DSAL and DRAL perform comparably. 
   
    \item {\bf Average Risk:} The second column in all the figures shows how the average risk (average of $L_d$) goes down with the number of steps ($t$). In all the cases, we see that the risk increases with increasing the value of $d$. We understand that the average risk of DSAL is higher than DRAL for Gisette and Phishing datasets and all values of $d$. For Guide dataset, DSAL always achieves lower risk compared to DRAL.
    
    For Gisette and Guide datasets, DSAL achieves similar risk as DSOL. For Phishing dataset, DSOL performs marginally better than DSAL and DRAL. DRAL does better risk minimization compared to DSOL for Phishing dataset. For Guide dataset, DRAL performs comparable to DSOL as $t$ becomes larger except for $d=0.1$. 
    
    \item {\bf Average Fraction of Asked Labels:} Third column in all the figures show the fraction of labels asked for a given time step $t$. We observe that the fraction of asked labels decreases with increasing $d$. For Gisette and Phishing datasets, DSAL asks significantly less number of labels as DRAL. This happens because DRAL asks labels every time in a specific region and completely ignores other regions, but DSAL asks labels in every region with some probability.
    For Guide dataset, the fraction of labels asked to become the same for both DSAL and DRAL as $t$ becomes larger.
    
    \item {\bf Average Fraction of Misclassified Examples:} The fourth column of all the figures, shows how the average fraction of misclassified examples goes down with $t$. We observe that the misclassification rate goes up with increasing $d$. We see that DRAL achieves a minimum average misclassification rate in all the cases compared to DSOL and DSAL except for the Guide dataset with $d=0.1$ value. For Gisette and Phishing datasets, DSAL achieves a comparable average misclassification rate compared to DSOL for all the cases. For Guide dataset, DSAL achieves a lower misclassification rate compared to DSOL except for $d=0.1$. 
    \item {\bf Average Fraction of Rejected Examples:} The fifth column in each figure shows how the rejection rate goes down with steps $t$. We see that the average fraction of rejected examples is higher in DRAL than DSAL and DSOL. Also, the rejection rate decreases with increasing $d$.
\end{itemize}
Thus, we see that the proposed active learning algorithms DRAL and DSAL effective reduce the number of labels required for learning the reject option classifier and perform better compared to online learning.

\section{Conclusion}
\label{section:conc}
In this paper, we have proposed novel active learning algorithms DRAL and DSAL. We presented mistake bounds for DRAL and convergence results for DSAL. We experimentally show that the proposed active learning algorithms reduce the number of labels required while maintaining a similar performance as online learning. 

\bibliography{Final-Biblography}
\bibliographystyle{aaai}

\appendix

\section{Kernelized Active Learning algorithm using Double Ramp Loss}
\label{app-sec:kernelized-DRAL}
In this section, we describe the kernel version of the active learning algorithm proposed in Algorithm~1 for learning nonlinear classifiers. We use the usual kernel trick to determine the classifier parameters. For every example presented in the algorithm, we maintain a variable $a_t$ and save it. The classifier after trial $t$ is represented as $f_{t}(\cdot) = \underset{s\leq t:Q_s=1}{\sum} a_s {\cal K}(\xx_s,\cdot)$. Thus, the algorithm has to maintain the $a_t$ values for all the examples for which the label was queried. The detailed algorithm is described as follows.
\begin{algorithm}[h]
\caption{Kernelized Active Learning algorithm using Double Ramp Loss}
\label{algo1}
\begin{algorithmic}
\State {\bf Input:} $d \in (0, 0.5)$, step size $\eta$, training set $S$
\State {\bf Output:} Weight vector $\ww$, Rejection width $\rho$
\State {\bf Initialize:} $a_{0}, \rho_{0}$
\For{$t = 1,..,T$} 
\State Sample $\xx_{t}$ from the training set $S$
\State Set $f_{t-1}(\xx) = \sum_{i=1}^{t-1} a_{i} \K (\xx_{i}, \xx)$
\If{$\rho_{t-1} -1 \leq |f_{t-1}(\xx_{t})| \leq \rho_{t-1} + 1$}
\State Set $Q_t=1$
\State Query the label $y_{t}$ of $\xx_t$
\If{($\rho_{t-1} - 1 \leq \; y_{t}f_{t-1}(\xx_{t}) \; \leq \rho_{t-1} + 1$)}
\State $a_{t} = \eta d y_{t}$ . 
\State $\rho_{t} = \rho_{t-1} - \eta d $
\ElsIf{($-\rho_{t-1} - 1 \leq y_{t}f_{t-1}(\xx_{t}) \leq -\rho_{t-1} + 1$)}
\State $a_{t} = \eta (1 - d) y_{t}$. 
\State $\rho_{t} = \rho_{t-1} + \eta (1-d)$. 
\EndIf
\Else 
\State $a_{t} = 0$. 
\State $\rho_{t} = \rho_{t-1}$.
\EndIf
\EndFor
\end{algorithmic}
\end{algorithm}

\section{Proof of Lemma~1}
\label{app-sec:lemma-1}
To prove Lemma 1, we will first prove Lemma 8 and Lemma 9.
\begin{lemma}
Assuming $\| \ww \| \leq \WW $ and $\| \xx \| \leq \RR$, Double ramp loss $L_{dr}$ satisfies following inequality. 
\begin{align*}
    L_{dr}(y ( \ww \cdot \xx ), \rho ) \geq m_{11} d + m_{12} d \left( \rho - y ( \ww \cdot \xx ) \right)
\end{align*}
\noindent where $m_{11} = m_{12} = m_{1} = \min \left( \frac{1}{ \rho }, \frac{2}{ d( 1 + \rho + \WW \RR ) } \right)$.
\end{lemma}
\begin{proof}
Assume 
\begin{align}
\label{eq:ldr-lower-bound1}
    L_{dr}(y ( \ww \cdot \xx ), \rho ) \geq m_{11} d + m_{12} d \left( \rho - y ( \ww \cdot \xx ) \right)
\end{align}
\noindent for some $m_{11}$ and $m_{12}$. We will prove that for $m_{11} = m_{12} = m_{1}$, eq.(\ref{eq:ldr-lower-bound1}) satisfies for all values of $y(\ww \cdot \xx)$ and $\rho$ of consideration.  It is easy to show that if eq.(\ref{eq:ldr-lower-bound1}) satisfies at $y( \ww \cdot \xx) = \rho + 1 $, $y( \ww \cdot \xx) = - \rho + 1$ and $y( \ww \cdot \xx) = - \WW \RR$ then eq.(\ref{eq:ldr-lower-bound1}) will satisfy for all $ y( \ww \cdot \xx ) $ for which $ \| \ww \| \leq \WW $ and $\| \xx \| \leq \RR$. At $ y( \ww \cdot \xx) = \rho + 1 $, eq.(\ref{eq:ldr-lower-bound1}) will be 
\begin{align}
\label{eq:lr-lower-bound11}
    \nonumber 0 &\geq m_{11} d + m_{12} d (-1) \\
    m_{12} &\geq m_{11}
\end{align}
\noindent At $ y( \ww \cdot \xx) = -\rho + 1 $, eq.(\ref{eq:ldr-lower-bound1}) will be
\begin{align}
\label{eq:lr-lower-bound12}
    2d \geq m_{11} d + m_{12} d (2 \rho - 1)
\end{align}
\noindent At $ y( \ww \cdot \xx ) = - \WW \RR $, eq.(\ref{eq:ldr-lower-bound1}) will be
\begin{align}
\label{eq:lr-lower-bound13}
    2 \geq m_{11} d + m_{12} d ( \rho + \WW \RR )
\end{align}
$m_{11} = m_{12} = m_{1} = \min \left( \frac{1}{ \rho }, \frac{2}{ d( 1 + \rho + \WW \RR ) } \right) $ will satisfy all three equations (i.e. eq.(\ref{eq:lr-lower-bound11}), eq.(\ref{eq:lr-lower-bound12}) and eq.(\ref{eq:lr-lower-bound13}) ).
\end{proof}

\begin{lemma}
Assuming $\| \ww \| \leq \WW $, $\| \xx \| \leq \RR$ and $\WW \RR > \rho$, Double ramp loss $L_{dr}$ satisfies following inequality. 
\begin{align*}
    L_{dr}(y ( \ww \cdot \xx ), \rho ) \geq m_{21}(1 + d) - m_{22} (1 - d) \left( \rho + y ( \ww \cdot \xx ) \right)
\end{align*}
\noindent where $m_{21} = \min \left( \frac{2 (2 \rho + 1) }{ (1+d)(\WW \RR + \rho + 1) }, \frac{1 + d(\WW \RR - \rho)}{ (1+d)( \WW \RR - \rho + 1 ) } \right)$ and $m_{22} = \max \left( \frac{2}{ (\WW \RR + \rho + 1) (1 - d) }, \frac{ (2-d)(\WW \RR - \rho) + 1 }{ (\WW \RR - \rho + 1)(\WW \RR - \rho)(1 - d) } \right) $.
\end{lemma}
\begin{proof}
Assume 
\begin{align}
\label{eq:ldr-lower-bound2}
    L_{dr}(y ( \ww \cdot \xx ), \rho ) \geq m_{21}(1 + d) - m_{22} (1 - d) \left( \rho + y ( \ww \cdot \xx ) \right)
\end{align}
\noindent for some $m_{21}$ and $m_{22}$. It is easy to show that if eq.(\ref{eq:ldr-lower-bound2}) satisfies at $y( \ww \cdot \xx) = \rho + 1 $, $y( \ww \cdot \xx) = - \rho + 1$ and $y( \ww \cdot \xx) = - \WW \RR$ then eq.(\ref{eq:ldr-lower-bound2}) will satisfy for all values of $ y( \ww \cdot \xx ) $ and $\rho$ of consideration. At $ y( \ww \cdot \xx) = \rho + 1 $, eq.(\ref{eq:ldr-lower-bound2}) will be 
\begin{align}
\label{eq:lr-lower-bound21}
    \nonumber 0 \geq & \; m_{21} (1 + d) - m_{22} (1 - d) (2\rho + 1) \\[5pt]
    m_{22}& \geq  \frac{m_{21} (1 + d) }{ (1 - d) (2 \rho + 1) }
\end{align}
At $ y( \ww \cdot \xx) = -\rho + 1 $, eq.(\ref{eq:ldr-lower-bound2}) will be
\begin{align}
\label{eq:lr-lower-bound22}
    \nonumber 2d &\geq m_{21} (1 + d) - m_{22} (1 - d) \\[5pt]
    m_{22} & \geq \frac{ m_{21} (1 + d) - 2d}{ (1 - d) }
\end{align}
At $ y( \ww \cdot \xx ) = - \WW \RR $, eq.(\ref{eq:ldr-lower-bound2}) will be
\begin{align}
\label{eq:lr-lower-bound23}
    \nonumber 2 &\geq m_{21} (1 + d) - m_{22} (\rho - \WW \RR) (1 - d) \\[5pt]
    m_{22} &\leq \frac{ 2 - m_{21} (1 + d) }{ (\WW \RR - \rho)(1 - d) }
\end{align}
\noindent One can check that $m_{21} = \min \left( \frac{ 2(2 \rho + 1) }{ (1 + d)( \WW \RR + \rho +1 ) }, \frac{ 1 + d(\WW \RR - \rho) }{ (1 + d)( \WW \RR - \rho + 1 ) } \right)$ and $m_{22} = \max \left( \frac{2}{ (\WW \RR + \rho + 1) (1 - d) }, \frac{ (2-d)(\WW \RR - \rho) + 1 }{ (\WW \RR - \rho + 1)(\WW \RR - \rho)(1 - d) } \right) $ satisfies eq.(\ref{eq:lr-lower-bound21}), eq.(\ref{eq:lr-lower-bound22}) and eq.(\ref{eq:lr-lower-bound23}).
\end{proof}

\noindent Now, we will prove Lemma 1 using Lemma 8 and Lemma 9. 
\begin{equation*}
\begin{aligned}
&\|  \ww_{t}  - {\alpha}\ww \|^{2} - {\| \ww_{t+1} - {\alpha}\ww \|}^{2}
={\| \ww_{t} - {\alpha}\ww \|}^{2}\\
&- {\| \ww_{t} + {\eta dy_{t}\xx_{t}}[\fc]} {+  {\eta (1-d)y_{t}\xx_{t}}[\sc] - {\alpha}\ww \|}^{2}
\end{aligned}
\end{equation*}
Note that only one of four indicator $C_{t}, R_{1t}, R_{2t}, M_{t}$ can be true at time $t$ therefore following equations will be true. 
\begin{equation*}
\begin{aligned}
&[\fc]^2 = [\fc]\\
&[\sc]^2 = [\sc]\\
&[\fc][\sc] = 0\\
\end{aligned}
\end{equation*}
Using above facts,
\begin{equation*}
\begin{aligned}
&\| \ww_{t}  - {\alpha}\ww \|^{2} - {\| \ww_{t+1} - {\alpha}\ww \|}^{2}\\
= \;\; &  {\| \ww_{t} \|}^{2} + {\alpha}^{2}{\| \ww \|}^{2} - 2{\alpha}({\ww \cdot \ww_{t}}) - \Big{[} { \| \ww_{t} \|}^{2} \\
&+ {\eta^2  {d}^2 {y_{t}}^2}{ \| \xx_{t} \| }^{2}[\fc]  \\
&+ {{ \eta^2 (1-d)}^2 {y_{t}}^2 }{ \| \xx_{t} \| }^{2}[\sc] \\ 
&+ {\alpha}^{2}{ \| \ww \| }^{2} + {2 \eta dy_{t}}({\ww_{t}} \cdot \xx_{t})[\fc]  \\
&+ {2 \eta (1-d)y_{t}}({\ww_{t}} \cdot \xx_{t})[\sc] - 2\alpha(\ww \cdot \ww_{t}) \\
& - {2{\alpha} \eta dy_{t}}(\ww \cdot \xx_{t})[\fc]  \\
&- {2{\alpha} \eta (1-d)y_{t}}(\ww \cdot \xx_{t})[\sc ]  \Big{]} \\[6pt]
=\;\;&  {2{\alpha} \eta dy_{t}}(\ww \cdot \xx_{t})[\fc] \\
&+ {2{\alpha} \eta (1-d)y_{t}}(\ww \cdot \xx_{t})[\sc] \\
&-{2 \eta dy_{t}}(\ww_{t} \cdot \xx_{t})[\fc] \\ 
&- {2 \eta (1-d)y_{t}}(\ww_{t} \cdot \xx_{t})[\sc] \\
&-{\eta^2 d^2{\| \xx_{t} \|}^2}[\fc] -{\eta^2 (1-d)^2{\| \xx_{t} \|}^2}[\sc] \end{aligned}
\end{equation*}
Combining the coefficient of $[\fc] \text{ and } [\sc]$,
\begin{equation*}
\begin{aligned}
\| &\ww_{t}- {\alpha}\ww \|^{2} - {\| \ww_{t+1} - {\alpha}\ww \|}^{2} 
=  {2{\alpha} \eta dy_{t}}(\ww \cdot \xx_{t})[ \fc ] \\
&+ {2{\alpha} \eta (1-d)y_{t}}(\ww \cdot \xx_{t})[ \sc ] 
- {2 \eta dy_{t}}(\ww_{t} \cdot \xx_{t})[ \fc ]\\
&- {2 \eta (1-d)y_{t}}(\ww_{t} \cdot \xx_{t})[ \sc ] -{ \eta^2 d^2{\| \xx_{t} \|}^2}[\fc ]\\
&-{ \eta^2 (1-d)^2{\| \xx_{t} \|}^2}[ \sc ] \\[6pt]
= \; & \; [ \fc ]\big{[} {2{\alpha} \eta dy_{t}}(\ww \cdot \xx_{t}) -{2 \eta dy_{t}}(\ww_{t} \cdot \xx_{t}) -{ \eta^2 d^2{\| \xx_{t} \|}^2} \big{]} \\
&+ [ \sc ]\big{[} {2{\alpha} \eta (1-d)y_{t}}(\ww \cdot \xx_{t}) \\
&-{2 \eta (1-d)y_{t}}(\ww_{t} \cdot \xx_{t}) 
-{ \eta^2 (1-d)^2{\| \xx_{t} \|}^2} \big{]}
\end{aligned}
\end{equation*}
Repeating the similar procedure for $\rho$, we get, 
\begin{equation*}
\begin{aligned}
&( {\rho}_{t} - {\alpha}\rho )^{2} - {( {\rho}_{t+1} - {\alpha}\rho )}^{2} 
={( {\rho}_{t} - {\alpha}\rho )}^{2} \\
&- ( {\rho}_{t} - { \eta d}[ \fc ]  + {\eta(1-d)}[ \sc ] - {\alpha}\rho )^{2} \\[6pt]
= \; & \; {  {\rho}_{t}  }^{2} + {\alpha}^2{\rho}^{2} - 2{\alpha}{\rho}{\rho}_{t} - \Big{[} {  {\rho}_{t}  }^{2} + \eta^2 {d^2}[ \fc ] \\
&+ {\eta^2 (1-d)^2}[ \sc ] + {\alpha}^2{\rho}^{2} - {2 \eta d {\rho}_{t}}[\fc] \\
&+ {2 \eta (1-d) {\rho}_{t}}[\sc] - 2{\alpha}{\rho}{\rho}_{t} + {2 {\alpha} \eta d \rho }[\fc] \\
&- {2{\alpha} \eta (1-d)  \rho}[\sc] \Big]\\[6pt]
=\; & \; [\fc] \thinspace \big{[} - \eta^2 d^2 + 2 \eta d{\rho}_{t} - 2 {\alpha} \eta d{\rho} \thinspace \big{]} \\
&+ [\sc]\big{[} - \eta^2 (1-d)^2 \\ 
&- 2 \eta (1-d) {\rho}_{t} + 2{\alpha} \eta (1-d)\rho   \big{]}
\end{aligned}
\end{equation*}
Adding $\| \ww_{t} - {\alpha}\ww \|^{2} - \| \ww_{t+1} - {\alpha}\ww \|^{2}$ and ${( {\rho}_{t} - {\alpha}\rho ) }^{2} - {( {\rho}_{t+1} - {\alpha}\rho )}^{2}$, we get the following.  
\begin{equation*}
\begin{aligned}
&\| \ww_{t} - {\alpha}\ww \|^{2} - \| \ww_{t+1} - {\alpha}\ww \|^{2} + {( {\rho}_{t} - {\alpha}\rho ) }^{2} - {( {\rho}_{t+1} - {\alpha}\rho )}^{2} \\
= \; &\ \; [\fc] \big{[} 2{\alpha} \eta d(y_{t}( \ww \cdot \xx_{t} ) - \rho) \\ 
&- 2 \eta d(y_{t}(\ww_{t} \cdot \xx_{t}) - {\rho}_{t}) - \eta^2 {d^2}( {\| \xx_{t} \|}^2 + 1)) \big{]} \\
&+ [\sc] \big{[} 2{\alpha} \eta (1-d)(y_{t}( \ww \cdot \xx_{t} ) + \rho )\\
&- 2 \eta (1-d)(y_{t}( \ww_{t} \cdot \xx_{t} ) + {\rho}_{t}) 
- \eta^2 (1-d)^2( {\| \xx_{t} \|}^2 + 1) \big{]}
\end{aligned}
\end{equation*}
If $\fc = 1$, then $L_{dr}( y_t (\ww_t \cdot \xx_t), \rho_t)=d[ \thinspace {\rho} + 1 - y_{t}(\ww_t \cdot \xx_{t}) \thinspace ]$. If $\sc = 1$, then $L_{dr}( y_t (\ww_t \cdot \xx_t), \rho_t)=2d + (1-d)[1 - y_{t}( \ww_t \cdot \xx_{t} ) - \rho]$. We use these facts and Lemma 8 and Lemma 9 to get the following.
\begin{equation*}
\begin{aligned}
&\| \ww_{t}- {\alpha}\ww \|^{2} - {\| \ww_{t+1} -} { {\alpha}\ww \|}^{2} + {( {\rho}_{t} - {\alpha}\rho ) }^{2} - {( {\rho}_{t+1} - {\alpha}\rho )}^{2} \\
\geq & \; [\fc] \Bigg{[} \frac{ 2\alpha \eta }{m_{1}} ( m_{1}d - L_{dr}(y_t ( \ww \cdot \xx_t),\rho_{t}) ) \\
&+ 2 \eta (L_{dr}( y_t (\ww_t \cdot \xx_t),\rho_{t}) - d) - \eta^2 d^2({\| \xx_{t} \|}^2 + 1) \Bigg{]}\\
&+ [\sc] \Bigg{[} \frac{2 \alpha \eta}{ m_{22} } ( m_{21}(1 + d) - L_{dr}( y_t(\ww \cdot \xx_t) ,\rho)  )\\
&+ 2 \eta ( L_{dr,\rho}( y_t(\ww_t \cdot \xx_t), \rho_t) - d - 1)
- \eta^2 (1-d)^2({\| \xx_{t} \|}^2 + 1) \Bigg{]}
\end{aligned}
\end{equation*}
Summing the above equation for all $t=1,2,...,T$.
\begin{equation*}
\begin{aligned}
&\sum\limits_{t=1}^{T} [\fc] \Bigg{[} \thinspace \frac{ 2{\alpha} \eta }{m_1}  ( m_1 d - L_{dr}( y_t(\ww \cdot \xx_t),\rho))\\
&+ 2 \eta (L_{dr}(y_t (\ww_t \cdot \xx_t), \rho_t) - d ) - \eta^2 d^2( {\| \xx_{t} \|}^2 + 1 ) \thinspace \Bigg{]} \\
&+ \sum\limits_{t=1}^{T} [\sc]\Bigg{[} \thinspace \frac{2{\alpha} \eta  }{ m_{22} }  \big( m_{21} (1 + d) - L_{dr}( y_{t}(\ww \cdot \xx_t), \rho_t)  \big) \\
&+ 2 \eta (L_{dr}( y_t ( \ww_t \cdot \xx_t ), \rho_t) - d - 1) - \eta^2 (1-d)^2( {\| \xx_{t} \|}^2 + 1 ) \thinspace \big{]} \\[6pt]
&\leq {\| \ww_{1} - {\alpha}\ww \|}^2 - {\| \ww_{T+1} - {\alpha}\ww \|}^2  +  {( {\rho}_{1} - {\alpha}\rho )}^2 - {( {\rho}_{T+1} - {\alpha}\rho )}^2  \\
&\leq  {\| \ww_{1} - {\alpha}\ww \|}^2 + {( {\rho}_{1} - {\alpha}\rho )}^2 \\
&=   {\alpha}^2{ \| \ww \| }^2 + (1 - {\alpha}\rho)^2
\end{aligned}
\end{equation*}
Here, we used the fact that we initialize with $\ww_1=\mathbf{0}$ and $\rho_1=1$. Rearranging terms, we will get required inequality.
\begin{align*}
   & \sum\limits_{t=1}^{T} [\fc]\big{[} \thinspace  2{\alpha} \eta d + 2 \eta (L_{dr}( y_t (\ww_t \cdot \xx_t),\rho_t)  - d ) \\ 
    &- \eta^2 d^2( {\| \xx_{t} \|}^2 + 1 ) \thinspace\big{]} + \sum\limits_{t=1}^{T} [\sc]\Bigg{[} \thinspace \frac{ 2{\alpha} \eta (1 + d) m_{21}}{ m_{22} }  \\
    &+ 2 \eta (L_{dr}( y_t ( \ww_t \cdot \xx_t), \rho_t) - d - 1) - \eta^2 (1-d)^2( {\| \xx_{t} \|}^2 + 1 ) \thinspace \Bigg{]} \\[6pt]
   & \leq {\alpha}^2{ \| \ww \| }^2 + (1 - {\alpha}\rho)^2 + \sum_{t=1}^T \frac{ 2 \alpha \eta }{ m_{1} }  L_{dr}( y_{t} (\ww \cdot \xx_t), \rho_t) [\fc] \\
   &+ \sum_{t=1}^T  \frac{2 \alpha \eta}{ m_{22} } L_{dr}( y_{t} ( \ww \cdot \xx_t ), \rho_{t}) [\sc] \\[6pt]
   & \leq  {\alpha}^2{ \| \ww \| }^2 + (1 - {\alpha}\rho)^2 + \sum_{t=1}^T \frac{ 2 \alpha \eta }{ m }  L_{dr}( y_{t} (\ww \cdot \xx_t), \rho_t)
\end{align*}
where $m = \min( m_{1}, m_{22} )$.

\section{Proof of Theorem~2}
\label{app-sec-theorem-2}
\begin{enumerate}
\item 
Putting $L_{dr}( y_{t} ( \ww \cdot \xx_t ),\rho)$ = 0 in the result of Lemma~1, we will get
\begin{equation}
\label{eq-theorem2-loss-zero}
\begin{aligned}
&\alpha^2  \| \ww \|^2 + (1 - \alpha \rho)^2 \\[6pt]
&\geq \sum\limits_{t=1}^{T} [\fc] \big[ \; 2{\alpha} \eta d + 2 \eta (L_{dr}( y_{t} ( \ww_{t} \cdot \xx_t), \rho_t) - d ) \\
&- \eta^2 d^2( {\| \xx_{t} \|}^2 + 1 ) \; \big] + \sum\limits_{t=1}^{T} [\sc] \Bigg[  \frac{ 2{\alpha} \eta (1 + d) m_{21} }{ m_{22} }  \\
&+ 2 \eta (L_{dr}( y_{t} ( \ww_{t} \cdot \xx_t ),\rho_t) - d - 1) - \eta^2 (1-d)^2( {\| \xx_{t} \|}^2 + 1 ) \Bigg]
\end{aligned}
\end{equation}
We choose the following value of $\alpha$. 
\begin{equation}
\label{eq-alpha-val-loss-zero-reject}
\begin{aligned}
\alpha = \max \begin{cases}
\frac{1 + \eta^2 d^2(R^2 + 1) + 2 \eta d}{2 \eta d}  \\
\frac{ m_{22} \left(1 + \eta^2 (1-d)^2(R^2 + 1) + 2 \eta (1-d) \right) }{2 m_{21} \eta (1+d)} 
\end{cases}
\end{aligned}
\end{equation}
This implies that $\alpha \geq \frac{1 + \eta^2 d^2(R^2 + 1) + 2 \eta d}{2 \eta d}$. Using this inequality in expression of coefficient of $\fc$ in eq.(\ref{eq-theorem2-loss-zero}),
\begin{equation}
\label{eq-alpha-value-proof21}
\begin{aligned}
2{\alpha} \eta d & + 2 \eta (L_{dr}( y_{t} ( \ww_{t} \cdot \xx_t ), \rho_t) - d ) - \eta^2 d^2( {\| \xx_{t} \|}^2 + 1 ) \\[6pt]
\geq   &  \;\; 2 \left( \frac{1 + \eta^2 d^2(R^2 + 1) + 2 \eta d}{2 \eta d} \right) \eta d \\[3pt]
&+ 2 \eta (L_{dr}( y_{t} ( \ww_{t} \cdot \xx_t ), \rho_t) - d ) - \eta^2 d^2( {\| \xx_{t} \|}^2 + 1 ) \\[6pt]
\geq &  \; 1 + \eta^2 d^2 (R^2 - \| \xx_{t} \|^2) + 2 \eta L_{dr}( y_{t} ( \ww_{t} \cdot \xx_t ), \rho_t) \\[6pt]
\geq & \; 1
\end{aligned}
\end{equation}
Moreover, from eq.(\ref{eq-alpha-val-loss-zero-reject}), we can say that $\alpha \geq \frac{ m_{22} \left( 1+ \eta^2 (1-d)^2 (R^2 + 1) + 2 \eta (1-d) \right) }{2 m_{21} \eta (1+d)}$. Using this inequality in coefficient of $R_{2t} + M_{t}$ in eq.(\ref{eq-theorem2-loss-zero}),
\begin{align*}
    &  \frac{ 2{\alpha} \eta  (1 + d) m_{21}}{ m_{22} }  + 2 \eta (L_{dr}( y_{t} ( \ww_{t} \cdot \xx_t ), \rho_t) - d - 1) \\[3pt]
     & \;\;\;\; - \eta^2 (1-d)^2( {\| \xx_{t} \|}^2 + 1 ) \\[6pt]
     \geq & \;  2 \left(  \frac{ 1 + \eta^2 (1-d)^2 (R^2 + 1) + 2 \eta (1-d) }{2 \eta (1+d)} \right) \eta (1+d) \\[5pt]
    & \;\; + 2 \eta (L_{dr}( y_{t} ( \ww_{t} \cdot \xx_t ), \rho_t) - d - 1) - \eta^2 (1-d)^2( {\| \xx_{t} \|}^2 + 1 ) \\[6pt]
     \geq & \; \eta^2(1-d)^2 (R^2 - \| \xx_{t} \|^2) \\[5pt] 
     &+ 2 \eta ( L_{dr}( y_{t} ( \ww_{t} \cdot \xx_t ), \rho_t) - 2d ) + 1 
\end{align*}
When $\sc=1$ then $L_{dr}( y_{t} ( \ww_{t} \cdot \xx_t ), \rho_t) \geq 2d$. Using this inequality,
\begin{equation}
\begin{aligned}
\label{eq-alpha-value-proof22}
     \frac{ 2{\alpha} \eta (1 + d) m_{21} }{ m_{22} } &+ 2 \eta (L_{dr}( y_{t} ( \ww_{t} \cdot \xx_t ), \rho_t) - d - 1) \\[4pt]
     &- \eta^2 (1-d)^2( {\| \xx_{t} \|}^2 + 1 ) \geq 1
\end{aligned}
\end{equation}    
Using eq.(\ref{eq-alpha-value-proof21}) and eq.(\ref{eq-alpha-value-proof22}), for $\alpha$ value given in eq.(\ref{eq-alpha-val-loss-zero-reject}),
\begin{align*}
     \sum_{t=1}^T [ R_{1t} + R_{2t} ] &\leq \sum_{t=1}^T \; [ C_{t} + R_{1t} ] + \sum_{t=1}^T \; [ R_{2t} + M_{t} ] \\[6pt]
     &\leq \alpha^2 \| \ww \|^2 + (1 - \alpha \rho)^2 \\
\end{align*}
Here $\alpha = \max \Bigg{(} \frac{1 + \eta^2 d^2(R^2 + 1) + 2 \eta d}{2 \eta d} ,\frac{ m_{22} \left( 1+ \eta^2 (1-d)^2(R^2 + 1) + 2 \eta (1-d) \right) }{2 \eta m_{21} (1+d)} \Bigg{)}$.
\item Putting $L_{dr}( y_{t} ( \ww \cdot \xx_t ), \rho) = 0,\;\forall t\in[T]$ in Lemma 1, we will get
\begin{equation}
\label{eq-theorem1-zero-loss}
\begin{aligned}
&{\alpha}^2 { \| \ww \| }^2 + (1 - {\alpha}\rho)^2 \\
&\geq \sum\limits_{t=1}^{T} [\fc][ \thinspace 2{\alpha} \eta d  + 2 \eta (L_{dr}( y_{t} ( \ww_{t} \cdot \xx_t ), \rho_t) - d ) \\ 
&- \eta^2 d^2( {\| \xx_{t} \|}^2 + 1 ) \thinspace] + \sum\limits_{t=1}^{T} [\sc] \Bigg[ \thinspace \frac{ 2{\alpha} \eta (1 + d) m_{21} }{ m_{22} }  \\
&+ 2 \eta (L_{dr}( y_{t} ( \ww_{t} \cdot \xx_t), \rho_t) - d - 1) - \eta^2 (1-d)^2( {\| \xx_{t} \|}^2 + 1 ) \thinspace \Bigg] \\[6pt]
\end{aligned}
\end{equation}
Now, take 
\begin{equation}
\label{eq-alpha-val-loss-zero-mistake}
\alpha = \max \begin{cases}
 \frac{ \eta d(R^2 + 1) + 2}{2 } \\
 \frac{ m_{22} \left( 1+ \eta^2 (1-d)^2(R^2 + 1) + 2 \eta (1-d) \right) }{2 \eta m_{21} (1+d)} 
 \end{cases}
\end{equation}
This implies that $\alpha \geq \frac{ \eta d(R^2 + 1) + 2}{2}$. Using this inequality in the expression of coefficient of $\fc$ in eq.(\ref{eq-theorem1-zero-loss}), 
\begin{equation}
\label{eq-coeff1-geq-zero}
\begin{aligned}
&2{\alpha} \eta  d + 2 \eta (L_{dr}( y_{t} ( \ww_{t} \cdot \xx_t ), \rho_t) - d ) - \eta^2  d^2( {\| \xx_{t} \|}^2 + 1 ) \\
&\geq 2 \left( \frac{ \eta d(R^2 + 1) + 2}{2} \right) \eta d + 2 \eta (L_{dr}( y_{t} (\ww_{t} \cdot \xx_t), \rho_t) - d) \\
&- \eta^2 d^2( {\| \xx_{t} \|}^2 + 1 )\\
 = & \; \eta^2 d^2( R^2 - {\| \xx_{t} \|}^2 ) + 2 \eta L_{dr}( y_{t} ( \ww_{t} \cdot \xx_t ), \rho_t) \\
& \geq  0, \;\forall t \in [T]
\end{aligned}
\end{equation}
Value of $\alpha$ in eq.(\ref{eq-alpha-val-loss-zero-mistake}) also implies that $\alpha \geq \frac{ m_{22} \left( 1+ \eta^2 (1-d)^2(R^2 + 1) + 2 \eta (1-d) \right) }{2 m_{21} \eta(1+d)}$. Using this inequality in the expression of coefficient of $\sc$ in eq.(\ref{eq-theorem1-zero-loss}), 
\begin{equation}
\label{eq-coeff2-geq-one}
\begin{aligned}
&  \frac{ 2\alpha \eta (1+d) m_{21} }{ m_{22} }  + 2 \eta ( L_{dr}( y_{t} (\ww_{t} \cdot \xx_t), \rho_t) - d - 1 ) \\[5pt]
&- \eta^2 (1-d)^2( {\| \xx_{t} \|}^2 + 1 ) \\[4pt]
&\geq  2\left( \frac{1+ \eta^2 (1-d)^2(R^2 + 1) + 2\eta(1-d)}{2 \eta (1+d)} \right) \eta (1+d) \\
&+ 2 \eta (L_{dr}( y_{t} ( \ww_{t} \cdot \xx_t ), \rho_t) - d -1) - \eta^2 (1-d)^2( {\| \xx_{t} \|}^2 + 1 ) \\
&= 1 + \eta^2 (1-d)^2( R^2 - {\| \xx_{t} \|}^2 ) \\ 
&+ 2 \eta (L_{dr}( y_{t} ( \ww_{t} \cdot \xx_t ), \rho_t) - 2d) \geq  1 
\end{aligned}
\end{equation}

From eq.(\ref{eq-coeff1-geq-zero}) and (\ref{eq-coeff2-geq-one}), we can say that using value of $\alpha$ given in eq.(\ref{eq-alpha-val-loss-zero-mistake}) will result into coefficient of $\fc$ greater than equal to 0 and coefficient of $\sc$ greater than equal to 1.
\begin{equation*}
\begin{aligned}
&\sum\limits_{t=1}^{T}  M_{t} \leq  \sum\limits_{t=1}^{T}[\sc]\\
&\leq \sum\limits_{t=1}^{T} [\sc] \Bigg[ \thinspace \frac{ 2{\alpha} \eta (1 + d) m_{21} }{ m_{22} }  + 2 \eta (L_{dr}( y_{t} ( \ww_{t} \cdot \xx_t ), \rho_t) \\ 
&- d - 1) - \eta^2 (1-d)^2( {\| \xx_{t} \|}^2 + 1 ) \thinspace ] \Bigg] \\
&\leq  {\alpha}^2{ \| \ww \| }^2 + (1 - {\alpha}\rho)^2 
\end{aligned}
\end{equation*}
\end{enumerate}

\section{Proof of Theorem 3}
\label{app-sec:theorem-3}
\begin{enumerate}
    \item 
According to lemma 1,
\begin{equation}
\begin{aligned}
\label{eq-theorem2-loss-nonzero}
&\sum\limits_{t=1}^{T}  [\fc]\big{[} \thinspace 2{\alpha} \eta d + 2 \eta (L_{dr}( y_{t} ( \ww_{t} \cdot \xx_t ), \rho_t) - d ) \\
&- \eta^2 d^2( {\| \xx_{t} \|}^2 + 1 ) \thinspace\big{]} + \sum\limits_{t=1}^{T} [\sc]\Bigg{[} \thinspace \frac{ 2{\alpha} \eta (1 + d) m_{21} }{ m_{22} }  \\
&+ 2 \eta (L_{dr}( y_{t} ( \ww_{t} \cdot \xx_t ), \rho_t) - d - 1) \\[5pt] 
&- \eta^2 (1-d)^2( {\| \xx_{t} \|}^2 + 1 ) \thinspace \Bigg{]} \\
&\leq {\alpha}^2{ \| \ww \| }^2 + (1 - {\alpha}\rho)^2 + \sum_{t=1}^T \frac{ 2 \alpha \eta }{ m } L_{dr}( y_{t} ( \ww \cdot \xx_t ), \rho)
\end{aligned}
\end{equation}

Now, taking value of $\alpha$ as 
\begin{equation}
\begin{aligned}
\label{eq-alpha-val-loss-nonzero-reject}
\alpha = \max \begin{cases}
\frac{1 + \eta^2 d^2(R^2 + 1) + 2 \eta d}{2 \eta d}  \\
\frac{ m_{22} \left( 1 + \eta^2 (1-d)^2(R^2 + 1) + 2 \eta (1-d) \right) }{2 \eta (1+d) m_{21} }
\end{cases}
\end{aligned}
\end{equation}
This implies that $\alpha \geq \frac{1 + \eta^2 d^2(R^2 + 1) + 2 \eta d}{2 \eta d}$. Using this inequality in expression of coefficient of $\fc$ in eq.(\ref{eq-theorem2-loss-nonzero}),
\begin{equation}
\label{eq-value-proof21}
\begin{aligned}
&2{\alpha} \eta  d + 2 \eta (L_{dr}( y_{t} ( \ww_{t} \cdot \xx_t ), \rho_t) - d ) - \eta^2 d^2( {\| \xx_{t} \|}^2 + 1 ) \\
&\geq  2 \left( \frac{1 + \eta^2 d^2(R^2 + 1) + 2 \eta d}{2 \eta d} \right) \eta d \\
&+ 2 \eta (L_{dr}( y_t ( \ww_{t} \cdot \xx_t), \rho_t) - d ) - \eta^2 d^2( {\| \xx_{t} \|}^2 + 1 ) \\
&\geq  1 + \eta^2 d^2 (R^2 - \| \xx_{t} \|^2) + 2 \eta L_{dr}( y_{t} ( \ww_{t} \cdot \xx_t ), \rho_t) \\
&\geq  1
\end{aligned}
\end{equation}
Moreover, from eq.(\ref{eq-alpha-val-loss-nonzero-reject}), we can say that $\alpha \geq \frac{ m_{22} (1+ \eta^2 (1-d)^2 (R^2 + 1) + 2 \eta (1-d)) }{2 \eta (1+d) m_{21} }$. Using this inequality in coefficient of $R_{2t} + M_{t}$ in eq.(\ref{eq-theorem2-loss-nonzero}),

\begin{align*}
   & \frac{ 2{\alpha} \eta  (1 + d) m_{21} }{ m_{22} }  + 2 \eta (L_{dr}( y_{t}(\ww_{t} \cdot \xx_t), \rho_t) - d - 1) \\[5pt]
    &- \eta^2 (1-d)^2( {\| \xx_{t} \|}^2 + 1 ) \\
    & \geq  2 \left(  \frac{ 1 + \eta^2 (1-d)^2 (R^2 + 1) + 2 \eta (1-d) }{2 \eta (1+d)} \right) \eta (1+d) \\[5pt]
    & + 2 \eta (L_{dr}( y_{t} ( \ww_{t} \cdot \xx_t ), \rho_t) - d - 1) - \eta^2 (1-d)^2( {\| \xx_{t} \|}^2 + 1 ) \\[5pt]
    & \geq \eta^2(1-d)^2 (R^2 - \| \xx_{t} \|^2) + 2 \eta ( L_{dr}( y_{t} ( \ww_{t} \cdot \xx_t), \rho_t) - 2d ) + 1\\[6pt]
\end{align*}
When $\sc=1$ then $L_{dr}( y_{t} ( \ww_{t} \cdot \xx_t), \rho_t) \geq 2d$. Using this inequality,
\begin{equation}
\begin{aligned}
\label{eq-alpha-value-proof23}
     & \frac{ 2{\alpha} \eta (1 + d) m_{21} }{ m_{22} }  + 2 \eta (L_{dr}( y_{t} ( \ww_{t} \cdot \xx_t ), \rho_t) - d - 1) \\[3pt]
     &- \eta^2 (1-d)^2( {\| \xx_{t} \|}^2 + 1 ) \geq 1
\end{aligned}
\end{equation}

Using eq.(\ref{eq-value-proof21}) and eq.(\ref{eq-alpha-value-proof23}), for $\alpha$ value given in eq.(\ref{eq-alpha-val-loss-nonzero-reject}),
\begin{align*}
     &\sum_{t=1}^T [ R_{1t} + R_{2t} ] \leq \sum_{t=1}^T \; [ C_{t} + R_{1t} ] + \sum_{t=1}^T \; [ R_{2t} + M_{t} ] \\
     & \leq \alpha^2 \| \ww \|^2 + (1 - \alpha \rho)^2 +  \sum_{t=1}^T  \frac{ 2\eta \alpha }{ m }  L_{dr} ( y_{t} ( \ww\cdot \xx_{t} ), \rho) \\
\end{align*}
Here, we used following value of $\alpha$.
\begin{align*}
    \alpha = \max \begin{cases}
      \frac{1 + \eta^2 d^2(R^2 + 1) + 2 \eta d}{2 \eta d} \\
      \frac{ m_{22} \left( 1+ \eta^2 (1-d)^2(R^2 + 1) + 2 \eta (1-d) \right) }{2 \eta (1+d) m_{21} } 
      \end{cases}
\end{align*}.

\item 
According to Lemma~1,
\begin{equation}
\label{eq-theorem1-nonzero-loss}
\begin{aligned}
&\sum\limits_{t=1}^{T}  [\fc]\big{[} \thinspace 2{\alpha} \eta d + 2 \eta (L_{dr}( y_{t} ( \ww_{t} \cdot \xx_t ), \rho_t) - d ) \\ 
&- \eta^2 d^2( {\| \xx_{t} \|}^2 + 1 ) \thinspace \big{]} + \sum\limits_{t=1}^{T} [\sc]\Bigg{[} \frac{ 2{\alpha} \eta (1 + d) m_{21} }{ m_{22} }  \\
&+ 2 \eta (L_{dr}( y_{t} ( \ww_{t} \cdot \xx_t ), \rho_t) - d - 1) - \eta^2 (1-d)^2( {\| \xx_{t} \|}^2 + 1 )  \Bigg{]} \\
&\leq {\alpha}^2{ \| \ww \| }^2 + (1 - {\alpha}\rho)^2 + \sum_{t=1}^T \frac{ 2 \alpha \eta }{ m }  L_{dr}( y_{t} ( \ww \cdot \xx_t), \rho_t)
\end{aligned}
\end{equation}
Now, take 
\begin{equation}
\label{eq-alpha-val-loss-nonzero-mistake}
\alpha = \max \begin{cases}
\frac{ \eta d(R^2 + 1) + 2}{2 } \\
\frac{ m_{22} ( 1+ \eta^2 (1-d)^2(R^2 + 1) + 2 \eta (1-d) ) }{2 \eta m_{21} (1+d)} 
\end{cases}
\end{equation}
This implies that $\alpha \geq \frac{ \eta d(R^2 + 1) + 2}{2}$. Using this inequality in the expression of coefficient of $\fc$ in eq.(\ref{eq-theorem1-nonzero-loss}), 
\begin{equation}
\label{eq-second-coeff1-geq-zero}
\begin{aligned}
&2{\alpha} \eta d    + 2 \eta (L_{dr}( y_{t} ( \ww_{t} \cdot \xx_t), \rho_t) - d ) - \eta^2 d^2( {\| \xx_{t} \|}^2 + 1 ) \\
& \geq 2 \Big{(} \frac{ \eta d(R^2 + 1) + 2}{2} \Big{)} \eta d + 2 \eta (L_{dr}( y_{t} (\ww_{t} \cdot \xx_t), \rho_t) - d) \\
 &-  \eta^2 d^2( {\| \xx_{t} \|}^2 + 1 )\\
 &=  \eta^2 d^2( R^2 - {\| \xx_{t} \|}^2 ) + 2 \eta L_{dr}( y_{t} ( \ww_{t} \cdot \xx_t ), \rho_t) \\
& \geq  0, \; \forall t \in [T]
\end{aligned}
\end{equation}
Value of $\alpha$ in eq.(\ref{eq-alpha-val-loss-nonzero-mistake}) also implies that $\alpha \geq \frac{ m_{22} \left( 1+ \eta^2 (1-d)^2(R^2 + 1) + 2 \eta (1-d) \right) }{2 \eta (1+d) m_{21} }$. Using this inequality in the expression of coefficient of $\sc$ in eq.(\ref{eq-theorem1-nonzero-loss}), 
\begin{equation}
\label{eq-second-coeff2-geq-one}
\begin{aligned}
& \frac{ 2\alpha \eta  (1+d) m_{21} }{ m_{22} }  + 2 \eta ( L_{dr}( y_{t} ( \ww_{t} \cdot \xx_t ), \rho_t) - d - 1 ) \\
&- \eta^2 (1-d)^2( {\| \xx_{t} \|}^2 + 1 ) \\
&\geq  2\Bigg{(} \frac{1+ \eta^2 (1-d)^2(R^2 + 1) + 2\eta(1-d)}{2 \eta (1+d)} \Bigg{)} \eta (1+d) \\
&+ 2 \eta (L_{dr}( y_{t} ( \ww_{t} \cdot \xx_t ), \rho_t) - d -1) \\
&- \eta^2 (1-d)^2( {\| \xx_{t} \|}^2 + 1 ) \\
&= 1 + \eta^2 (1-d)^2( R^2 - {\| \xx_{t} \|}^2 ) \\
&+ 2 \eta (L_{dr}( y_{t} ( \ww_{t} \cdot \xx_t), \rho_t) - 2d) \\[4pt]
& \geq  1
\end{aligned}
\end{equation}
From eq.(\ref{eq-second-coeff1-geq-zero}) and (\ref{eq-second-coeff2-geq-one}), we can say that using value of $\alpha$ given in eq.(\ref{eq-alpha-val-loss-nonzero-mistake}) will result into coefficient of $\fc$ greater than equal to 0 and coefficient of $\sc$ greater than equal to 1.
\begin{equation*}
\begin{aligned}
\sum\limits_{t=1}^{T} & M_{t} \leq  \sum\limits_{t=1}^{T}[\sc] \\ 
\leq &  \sum\limits_{t=1}^{T} [\sc]  \Bigg[ \thinspace \frac{ 2{\alpha} \eta (1 + d) m_{21} }{ m_{22} }  + 2 \eta (L_{dr}( y_{t} ( \ww_{t} \cdot \xx_t ), \rho) \\
&- d - 1) - \eta^2 (1-d)^2( {\| \xx_{t} \|}^2 + 1 ) \thinspace \Bigg] \\
\leq & \; {\alpha}^2{ \| \ww \| }^2 + (1 - {\alpha}\rho)^2 +  \sum_{t=1}^T \frac{ 2\eta \alpha }{ m }  L_{dr}( y_{t} ( \ww\cdot \xx_{t} ), \rho)
\end{aligned}
\end{equation*}

\end{enumerate}

\section{Proof of Lemma 4}
\label{app-sec:lemma-4}
To prove $\beta-$smoothness property of $L_{ds}(yf(\xx), \rho)$, we will first get Hessian matrix of $L_{ds}(yf(\xx), \rho)$ (i.e. $\nabla^2 L_{ds}(yf(\xx), \rho)$). 
\begin{equation}
\label{eq:def-yfx}
\begin{aligned}
&\frac{ \partial L_{ds}(yf(\xx), \rho ) }{ \partial \ww }  = -2d\gamma y \xx \big[ \ms ( 1 - \ms )  \big]  \\[3pt]
&\;\;\;-2(1-d)\gamma y \xx \big[ \ps ( 1 - \ps ) \big]
\end{aligned}
\end{equation}

\begin{equation}
\label{eq:def-rho}
\begin{aligned}
&\frac{ \partial L_{ds}(yf(\xx), \rho) }{ \partial  \rho  }  = 2d\gamma \big[ \ms (1 - \ms)  \big]  \\[3pt]
&\;\;\;-2(1-d)\gamma \big[ \ps (1 - \ps) \big]
\end{aligned}
\end{equation}

Now, taking double derivative of $L_{ds} ( y f(\xx), \rho)$,
\begin{align}
  \nonumber   &\frac{ \partial^2 L_{ds} }{ \partial \ww^2 } = 2 d \gamma^2 \xx \xx^{T} \Big[ \ms ( 1 - \ms ) \\
  \nonumber  & -2  \sigma^2(  yf(\xx)  - \rho ) (1 - \ms) \Big] \\
  \nonumber  & +2 (1-d)  \gamma^2 \xx \xx^{T} \Big[ \ps (1 - \ps) \\ 
    & - 2 \sigma^2   ( yf(\xx) + \rho)(1 - \ps) \Big] \label{eq:double-def-yfx-yfx}
    \end{align}

\begin{equation}
\label{eq:double-def-yfx-rho}
    \begin{aligned}
    \frac{ \partial^2 L_{ds} }{ \partial \ww  \; \partial \rho } &= 2 d \gamma^2 y \xx \Big[ - \ms ( 1 - \ms ) \\
    & +2 \sigma^2( yf(\xx) - \rho ) (1 - \ms) \Big] \\
    & +2 (1-d) \gamma^2 y \xx \Big[ \ps (1 - \ps) \\ 
    & - 2\sigma^2 ( yf(\xx) + \rho ) (1 - \ps) \Big]
    \end{aligned}
\end{equation}

\begin{equation}
\label{eq:double-def-rho-yfx}
    \begin{aligned}
    \frac{ \partial^2 L_{ds} }{ \partial \rho \; \partial \ww  } &= 2 d \gamma^2  y \xx \Big[ -\ms ( 1 - \ms ) \\
    & +2 \sigma^2( yf(\xx) - \rho ) (1 - \ms) \Big] \\
    & +2 (1-d) \gamma^2 y \xx \Big[ \ps (1 - \ps) \\ 
    & - 2\sigma^2 ( yf(\xx) + \rho ) (1 - \ps) \Big]
    \end{aligned}
\end{equation}

\begin{equation}
\label{eq:double-def-rho-rho}
    \begin{aligned}
    \frac{ \partial^2 L_{ds} }{ \partial \rho^2 } =& 2 d \gamma^2 \Big[ \ms ( 1 - \ms ) \\
    & -2 \sigma^2( yf(\xx) - \rho ) (1 - \ms) \Big] \\
    & +2 (1-d) \gamma^2 \Big[ \ps (1 - \ps) \\ 
    & - 2\sigma^2 ( yf(\xx) + \rho ) (1 - \ps) \Big]
    \end{aligned}
\end{equation}

Using eq.(\ref{eq:double-def-yfx-yfx}), (\ref{eq:double-def-yfx-rho}), (\ref{eq:double-def-rho-yfx}) and (\ref{eq:double-def-rho-rho}), we can construct Hessian matrix $\nabla^2 L_{ds} (yf(\xx), \rho)$.

\begin{equation}
    \begin{aligned}
    \nabla^2 L_{ds} ( yf(\xx), \rho ) = 
    \begin{pmatrix}
    \frac{ \partial^2 L_{ds}(yf(\xx), \rho) }{ \partial \ww^2 } & \frac{ \partial^2 L_{ds}(yf(\xx), \rho) }{ \partial \ww \; \partial \rho }\\[7pt]
    \frac{ \partial^2 L_{ds}(yf(\xx), \rho) }{ \partial \rho \; \partial \ww } & \frac{ \partial^2 L_{ds}(yf(\xx), \rho) }{ \partial \rho^2 }
    \end{pmatrix}
    \end{aligned}
\end{equation}

We know that upper bound on the spectral norm of the Hessian matrix $\nabla^2 L_{ds} ( yf(\xx), \rho )$ is smoothness constant $\beta$ of $L_{ds} ( yf(\xx), \rho )$. To upper bound the spectral norm of the Hessian matrix, we use following inequality. 
\begin{equation}
\label{eq:l2-frob-norm-inequality}
    \| \nabla^2  L_{ds} ( yf(\xx), \rho ) \|_{2} \leq \| \nabla^2 L_{ds} ( yf(\xx), \rho ) \|_{F}
\end{equation}
where $\|.\|_{2}$ stands for spectral norm and $\|.\|_{F}$ stands for Frobenius norm.
We can think $\ms ( 1 - \ms ) - 2 \sigma^2( yf(\xx) - \rho ) (1 - \ms)$ as a cubic polynomial in $\ms$. Now, using fact that $\ms \in [0, 1]$, we get range of the polynomial as [-0.1, 0.1]. In the same manner, we can get range of $\ps ( 1 - \ps ) - 2 \sigma^2( yf(\xx) + \rho ) (1 - \ps)$ as [-0.1, 0.1] therefore,  We can use $| \ms ( 1 - \ms ) - 2 \sigma^2( yf(\xx) - \rho ) (1 - \ms) | \leq 0.1 $ and $| \ps ( 1 - \ps ) - 2 \sigma^2( yf(\xx) + \rho ) (1 - \ps) | \leq 0.1 $. Using $\| \xx \| \leq \RR$, we get that 
\begin{equation}
\label{eq:frob-norm-bound}
    \begin{aligned}
    \| \nabla^2 L_{ds} ( yf(\xx), \rho ) \|_{F} \leq \frac{\gamma^2}{5} \big[ \RR^2 + 1 \big]
    \end{aligned}
\end{equation}
Using eq.(\ref{eq:l2-frob-norm-inequality}) and eq.(\ref{eq:frob-norm-bound}), we can say that $L_{ds} ( yf(\xx) , \rho)$ is $\beta-$smooth with smoothness constant $\beta = \frac{ \gamma^2 }{5} \big[ \RR^2 + 1 \big] $.

\section{Proof of Theorem 6}
\label{app-sec:theorem-6}
Let $\Theta = [\ww \;  \rho]$. Using smoothness property of $L_{ds}(y_{t}f_{t}(\xx_{t}), \rho_{t} )$, 

\begin{equation}
    \begin{aligned}
    L_{ds}& ( y_{t}f_{t+1}( \xx_{t} ), \rho_{t}) - L_{ds}( y_{t}f_{t}( \xx_{t} ), \rho_{t} ) \\ 
    &\leq \Big( \nabla_{\P} L_{ds}( y_{t} f_{t} (\xx_{t}), \rho_{t} ) \cdot  \big( \Theta_{t+1} - \Theta_{t} \big) \Big) \\ 
    &+ \frac{\beta}{2} \| \Theta_{t+1} - \Theta_{t} \|^2 
    \end{aligned}
\end{equation}

We know that $\Theta_{t+1} - \Theta_{t} = - \eta z_{t} \nabla_{\Theta} L_{ds}( y_{t} f_{t} (\xx_{t}), \rho_{t} )$ where $\eta$ is step-size. As $z_{t} \in \{0, 1\}$, we use fact that $z_{t}^2 = z_{t}$.

\begin{equation*}
    \begin{aligned}
    & L_{ds}  (  y_{t} f_{t+1} (\xx_{t}), \rho_{t} ) - L_{ds}( y_{t} f_{t} (\xx_{t}), \rho_{t} ) \\[3pt]  
    & \;\; \leq  \; - \eta z_{t} \| \nabla L_{ds}(y_{t} f_{t} (\xx_{t}), \rho_{t}) \|^2 + \frac{\beta z_{t} \eta^2 }{2} \|  \nabla L_{ds}( y_{t} f_{t} (\xx_{t}), \rho_{t} ) \|^2
    \end{aligned} 
\end{equation*}

Taking $\E_{z}$ on both side, we will get following equation.  
\begin{equation*}
    \begin{aligned}
    & \E_{z} \Big[ L_{ds}  (  y_{t} f_{t+1} (\xx_{t}), \rho_{t} ) - L_{ds}( y_{t} f_{t} (\xx_{t}), \rho_{t} ) \Big] \\[3pt]  
    & \;\; \leq  \; \Big( - \eta + \frac{\beta  \eta^2 }{2} \Big) \|  \nabla L_{ds}( y_{t} f_{t} (\xx_{t}), \rho_{t} ) \|^2 \E_{z}[ z_{t} ]
    \end{aligned} 
\end{equation*}

Using $\E_{z}[ z_{t} ] = p_{t} \leq 1$,
\begin{equation*}
    \begin{aligned}
    & \E_{z} \Big[ L_{ds}  (  y_{t} f_{t+1} (\xx_{t}), \rho_{t} ) - L_{ds}( y_{t} f_{t} (\xx_{t}), \rho_{t} ) \Big] \\[3pt]  
    & \;\; \leq  \; \Big( - \eta + \frac{\beta  \eta^2 }{2} \Big) \|  \nabla L_{ds}( y_{t} f_{t} (\xx_{t}), \rho_{t} ) \|^2
    \end{aligned} 
\end{equation*}

Multiplying above equation with -1 on both side,

\begin{equation*}
    \begin{aligned}
    &  \Big( \eta  - \frac{\beta \eta^2 }{2} \Big)  \| \nabla L_{ds}(y_{t} f_{t} (\xx_{t}), \rho_{t}) \|^2  \\[5pt]
    & \;\; \leq   \; \E_{z} \Big[ L_{ds}  (  y_{t} f_{t+1} (\xx_{t}), \rho_{t} ) - L_{ds}( y_{t} f_{t} (\xx_{t}), \rho_{t} ) \Big] \\[3pt]   
    & \;\;\ \leq \; \E_{z} \Big[ L_{ds}  (  y_{t+1} f_{t+1} (\xx_{t+1}), \rho_{t+1} ) - L_{ds}( y_{t} f_{t} (\xx_{t}), \rho_{t} ) \\[4pt] 
    &+ L_{ds}  (  y_{t} f_{t+1} (\xx_{t}), \rho_{t} ) - L_{ds}  (  y_{t+1} f_{t+1} (\xx_{t+1}), \rho_{t} )  \Big] \\[4pt]
    & \;\; \leq \; \E_{z} \Big[ L_{ds}  (  y_{t+1} f_{t+1} (\xx_{t+1}), \rho_{t+1} ) - L_{ds}( y_{t} f_{t} (\xx_{t}), \rho_{t} ) + 2 \Big]
    \end{aligned} 
\end{equation*}

Taking sum over $t=1,...,T$, we will get
\begin{equation}
    \sum_{t=1}^T \| \nabla L_{ds}( y_{t}f_{t}(\xx_{t}), \rho_{t} ) \|^2 = \mathcal{R}(T) \leq \frac{2T + 2}{ \eta - \frac{\beta \eta^2}{2} }
\end{equation}

Taking $\eta = \frac{1}{\beta}$, we will get
\begin{equation}
    \sum_{t=1}^T \| \nabla L_{ds}( y_{t}f_{t}(\xx_{t}), \rho_{t} ) \|^2 = \mathcal{R}(T) \leq 4 \beta ( T + 1 )
\end{equation}

Using $\beta = \frac{\gamma^2}{5} (\RR^2 + 1)$, we will get following local regret. 

\begin{equation}
    \mathcal{R}(T) \leq \frac{4 \gamma^2}{5} (\RR^2 +1)(T + 1) 
\end{equation}

\section{Active Learning of Non-Linear Reject Option Classifiers Based on Double Sigmoid Loss}
\label{app-sec:DSAL-kernel}

The query probability function and $\rho-$update equation of non-linear double sigmoid active learning algorithm is same as query update of the linear algorithm. They are as follows. 
\begin{align}
    \label{eq:proper-bimodal-probability}
    \text{probability } p_{t} = 4 \; \sigma ( | f_{t}( \xx_{t} ) | - \rho_{t} )\left( 1 - \sigma( | f_{t} (\xx_{t}) |  - \rho_{t}) \right)
\end{align}

\begin{align}
    \label{eq:proper-bimodal-rho-update}
     \nonumber \text{Update}(\rho) &= - 2\alpha \Big{[}d  \sigma ( y_{t} f_{t}( \xx_{t} ) - \rho_{t} )\left( 1 - \sigma( y_{t} f_{t}( \xx_{t} )  - \rho_{t}) \right)\\
    & \;\;\;\;\;- (1 - d)  \sigma ( y_{t} f_{t}( \xx_{t} ) + \rho_{t} )\left( 1 - \sigma( y_{t} f_{t}( \xx_{t} ) + \rho_{t}) \right)  \Big{]}
\end{align}{}

\begin{algorithm}
\begin{algorithmic}
\caption{Kernalized active learning algorithm using Double sigmoid loss function}
\State {\bf Input:} $d \in (0, 0.5)$, step size $\eta$. 
\State {\bf Output:} Weight vector $\ww$, Rejection width $\rho$.
\State {\bf Initialize:} $\ww_{1}, \rho_{1}$
\For{$t = 1,..,T$} 
\State Sample $\xx_{t} \in \mathbb{R}^{d}$
\State Set $f_{t-1}(\xx) =  \sum_{i=1}^{t-1} a_{i} \K(\xx_{i}, \xx) $ 
\State Set $p_{t}$ using eq.(\ref{eq:proper-bimodal-probability}).  Draw a Bernouli random variable $z_{t} \in \{0, 1\}$ of parameter $p_{t}$.
\If{$z_{t} == 1$}
\State Query for correct label $y_{t}$ of $\xx_t$.
\State $a_{t} = 2y_{t}  \alpha \big[  d\sigma ( y_{t} f_{t-1}( \xx_{t} )  - \rho_{t} )\left( 1 - \sigma( y_{t} f_{t-1} ( \xx_{t} )   - \rho_{t}) \right) + (1 - d)\sigma ( y_{t} f_{t-1} ( \xx_{t} )  + \rho_{t} )\left( 1 - \sigma( y_{t} f_{t-1} ( \xx_{t} )   + \rho_{t}) \right)  \big]
$
\State Update $\rho_{t}$ using eq.(\ref{eq:proper-bimodal-rho-update}). ($\rho_{t+1} = \rho_{t} + \text{Update}(\rho))$)
\Else{}
\State $a_{t} = 0$.
\State $\rho_{t+1} = \rho_{t}$.
\EndIf
\EndFor
\end{algorithmic}
\end{algorithm}

\end{document}